\documentclass{article}

\usepackage{fullpage}
\usepackage{amsmath,amssymb,amsthm,mathtools}
\usepackage{booktabs}
\usepackage{graphicx}
\usepackage{float}
\usepackage{microtype}
\usepackage{url}
\usepackage{natbib}
\usepackage{xcolor}
\usepackage{hyperref}

\setcitestyle{authoryear,round,citesep={;},aysep={,},yysep={;}}
\graphicspath{{./}}

\hypersetup{
  colorlinks=true,
  citecolor=blue,
  linkcolor=blue,
  urlcolor=blue,
  pdfauthor={Vicente Opazo},
  pdftitle={Three Tokens Force Exponential Feature Rank in Nonnegative Kernel Attention}
}

\newtheorem{theorem}{Theorem}

\newcommand{\bits}{\{0,1\}}
\newcommand{\ip}[2]{\langle #1,#2\rangle}
\newcommand{\cG}{\mathcal{G}}
\newcommand{\cC}{\mathcal{C}}
\newcommand{\cQ}{\mathcal{Q}}

\newcommand{\witnessbox}[1]{\boxed{\makebox[0.75em]{$\vphantom{xyz}#1$}}}

\title{Three Tokens Force Exponential Feature Rank\\
in Nonnegative Kernel Attention}

\author{
  Vicente Opazo\\
  CENIA\\
  \texttt{vicente.opazo@cenia.cl}
}
\date{}

\begin{document}

\maketitle

\begin{abstract}
How much feature rank does comparison require in kernel attention?  On
Min-IP over $m$-bit tokens, rank one solves every sequence of length at most
two exactly.  At length three, the minimum feature rank of one normalized
nonnegative kernel-attention head is $2^{\Theta(m)}$ for error strictly below $1/2$ on every
input, even with arbitrary finite-dimensional tokenwise values and
query-dependent affine readouts.  Dense softmax solves this three-token
task with $m$-dimensional scores and temperature constant in $m$.
For every fixed number of heads $H$, the minimum total feature rank is
$2^{\Theta_H(m)}$ for the same error guarantee at exact length $H+2$ in
one attention layer with affine mixing.  These bounds also hold with
position-dependent maps and a final causal query.
For one head with polynomial readout of fixed degree at most $D$, rank one
suffices at exact length $D+1$, while exact length $D+2$ requires exponential
feature rank.  With unrestricted exact-real decoding, a scalar rank-one
construction solves the task at every finite length.  This motivates a
separate bound on total communication for deterministic models with
finite-alphabet cross-token channels and any number of heads and layers.
In this setting, correctness up to length $n$ requires $\Omega(n\log m)$
bits over a range of lengths that grows exponentially with $m$.
\end{abstract}

\section{Introduction}

One way to study the expressive power of attention is to isolate the basic
operations that an attention layer makes available to a larger computation.
Selection is a basic use of attention: a query compares candidate tokens
and uses their scores to decide which information to pass on.
Nearest-neighbor and minimum-inner-product tasks
have therefore been used to compare attention architectures, as
seen in \citet{zaheer2020bigbird}, \citet{amsel2025quality}, and
\citet{alman2025fundamental}.

Full self-attention makes these comparisons directly: each query compares
itself directly with every token in a sequence~\citep{vaswani2017attention}.
This computes $N^2$ pairwise scores.  Kernelized linear
attention instead factors the attention kernel and summarizes the sequence
in a key--value matrix and a key sum.
Examples include normalized FAVOR+ with fixed sampled
features~\citep{choromanski2021performers} and ELU+1
attention~\citep{katharopoulos2020transformers}.  Factoring avoids
forming the pairwise matrix, but every query now receives the rest of the
sequence through a shared, fixed-dimensional sketch.  We ask whether this
shared summary supports comparison with an affine readout.

We isolate this question through \emph{Min-IP}.  Given a sequence
$X=(x_1,\ldots,x_N)$ of Boolean tokens $x_i\in\bits^m$, the target at every
position is
\begin{equation}
  t_i(X)=\min_{1\le j\le N}\ip{x_i}{x_j}.
  \label{eq:task-intro}
\end{equation}
For Boolean tokens, the inner product counts coordinates where both tokens
are one.  We call this their \emph{overlap}.
The output is just the minimum score.  The model does not need to identify
or copy a token that attains it.  Self-token interactions, repeated
tokens, and ties are included.  The task is inspired by Orthogonal Vectors
because $t_i=0$ exactly when $x_i$ has an orthogonal partner in the sequence,
but our lower bounds follow from the architecture alone and use no
computational hardness assumption.

Linear attention is often motivated as a long-context alternative to full
attention.  We show that the difficulty already appears with two competing
candidates.  At length two, each query has only one cross-token candidate,
and a rank-one normalized average with an affine readout recovers the answer
exactly.  At length three, comparison between two candidates requires
$2^{\Theta(m)}$ kernel features.  Longer contexts strengthen this lower bound.

We measure \emph{kernel feature rank}.  Our dense softmax solution uses
$m$-dimensional scores, yet
its Boolean-domain kernel has rank $2^m$.  The lower bound applies to \emph{any} nonnegative
kernel supporting Min-IP, with arbitrary finite-dimensional tokenwise
values and a query-dependent affine readout.

The proof uses only three families of inputs of the same length.  An exact
linear identity cancels all value contributions and forces a constant
multiplicative preference whenever two candidate overlaps differ by a fixed
gap.  Repeating the comparison at intermediate Boolean overlaps amplifies this preference
exponentially.  A constant-weight code and a nonsymmetric
approximate-identity lemma then turn it into a rank lower bound.

For several heads, we first multiply by their denominators and take a
finite difference that cancels all value terms.  This forces one head to
prefer a better key over every sufficiently worse key.  That head may change
between comparisons.  The amplification still bounds the sum of the head ranks.

We also study polynomial readouts.  For one head and any fixed maximum
degree $D$, rank one suffices at exact length $D+1$, while
exact length $D+2$ requires exponential rank.  The rank-one construction uses
$2^m$ value coordinates.

With unrestricted exact-real decoding,
one scalar rank-one head solves Min-IP at every finite length using
length-independent maps.  For broader deterministic models with arbitrary
local computation and multiple layers, we give a separate bound
on communication with finite alphabets, counting every cross-token channel.

\paragraph{Contributions.} Our contributions are:
\begin{itemize}
  \item The sharp length-two/length-three transition, with an explicit
  lower-bound exponent and extensions to positional maps, a final
  causal query, and distinct tokens with unique minima.
  \item An exponential total-feature-rank bound for fixed $H$ heads at
  length $H+2$ with affine mixing, and an exact relation between polynomial
  readout degree and the one-head context threshold.
  \item A deterministic $q\log_2(s+1)$ finite-transcript bound for multiple
  heads and layers, where a length-$2q$ family encodes $q$ independent answers
  in $\{0,\ldots,s\}$.
\end{itemize}
Three controlled experiments examine the proof mechanism, learned feature
capacity, and head allocation at a fixed feature budget.

\section{Task and architecture}
\label{sec:setup}

\subsection{Min-IP}

Fix integers $m,n\ge1$.  Let $\mathcal X_{m,n}$ be the set of sequences
$X=(x_1,\ldots,x_N)$ with $1\le N\le n$ and tokens $x_i\in\bits^m$.  A model
solves the task with error $\varepsilon$ if it outputs real scalars
$\widehat t_i(X)$ satisfying
\begin{equation}
  \left|\widehat t_i(X)-t_i(X)\right|<\varepsilon
  \qquad
  \forall X\in\mathcal X_{m,n},\quad
  \forall i\in\{1,\ldots,|X|\}.
  \label{eq:strict-error}
\end{equation}
We focus on $0<\varepsilon\le1/2$.  Since $t_i$ is an
integer and the inequality is strict, nearest-integer rounding uniquely
recovers every target.  Let $r^\star(m,n)$ denote the
minimum feature dimension of a single architecture that succeeds at
all lengths up to $n$ with error $<1/2$.

\subsection{Normalized nonnegative kernel attention}
\label{sec:kernel-model}

We study one head in one attention layer.  Its query and key maps
$\phi_Q,\phi_K:\bits^m\to\mathbb{R}^r$ define the kernel
\[
  \alpha(x,z)=\ip{\phi_Q(x)}{\phi_K(z)}\ge0
\qquad\text{on the Boolean domain.}
\]
The value map $v:\bits^m\to\mathbb{R}^{d_v}$ is arbitrary, with any finite
$d_v$.  For query $x_i$, the head takes a weighted average over all tokens:
\begin{equation}
  a_i(X)=
  \frac{\sum_j\alpha(x_i,x_j)v(x_j)}
       {\sum_j\alpha(x_i,x_j)}.
  \label{eq:kernel-attention}
\end{equation}
The denominator must be positive on every valid input.  An affine readout
gives the scalar prediction
\begin{equation}
  \widehat t_i=\beta(x_i)+w(x_i)^\top a_i(X).
  \label{eq:affine-readout}
\end{equation}
The bias $\beta(x_i)$ and weights $w(x_i)\in\mathbb R^{d_v}$ can be arbitrary
functions of the query.  Only the kernel weights must be nonnegative.
Feature coordinates, values, and readout coefficients may have either sign.

All maps are tokenwise: each uses only its input token.  The normalized
average is the only communication between tokens.  Arithmetic is exact,
and maps do not use positions unless stated otherwise.  For each input
dimension and fixed sequence length, the lower bounds allow a separate
choice of maps, so they also cover maps shared across lengths.

All queries use the same sequence summary:
\[
  S=\sum_j\phi_K(x_j)v(x_j)^\top\in\mathbb{R}^{r\times d_v},
  \qquad
  z=\sum_j\phi_K(x_j)\in\mathbb{R}^r.
\]
The attention output is $a_i=S^\top\phi_Q(x_i)/[\phi_Q(x_i)^\top z]$.

For the rank argument, write $[m]=\{1,\ldots,m\}$ and let
$x_U\in\bits^m$ denote the indicator vector of $U\subseteq[m]$.
The Boolean-domain kernel matrix has entries $A_{U,V}=\alpha(x_U,x_V)$.
Putting the query and key feature vectors in the rows of $\Phi_Q$ and
$\Phi_K$ gives $A=\Phi_Q\Phi_K^\top$, so $\operatorname{rank}(A)\le r$.
Thus $r$ measures the dimension used to factor the \emph{kernel},
not the rank of the query/key score matrix before exponentiation.

\section{A three-token exponential separation}
\label{sec:rank}

The third token introduces a choice between two candidates.  We first state
what this costs in the one-head model of Section~\ref{sec:kernel-model},
then explain why correctness forces the attention weights to make that choice.

\begin{theorem}[Three-token phase transition]
\label{thm:three-token}
For every $m\ge1$,
\[
 r^\star(m,1)=r^\star(m,2)=1.
\]
For every $m\ge168$, every one-head model
correct on all sequences of length exactly three with error strictly below
$1/2$ satisfies
\begin{equation}
 r\ge2^{m/10-6}.
 \label{eq:three-token-lower}
\end{equation}
As $m$ grows,
\[
 r\ge2^{(c_3-o(1))m},
 \qquad c_3=1-\log_2(1+4^{-1/5})\approx0.19.
\]
An exact positive-feature construction gives $r^\star(m,3)\le2^m$,
so $r^\star(m,3)=2^{\Theta(m)}$.
\end{theorem}

For lengths at most two, the uniform kernel with $v(z)=z$ and readout
$2\ip{x_i}{a_i(X)}-\|x_i\|_1$ is exact.  At length two, there is only one
candidate besides the query, and its overlap is never larger than the query's self-overlap.
The context-length threshold is sharp: two tokens admit rank one, whereas
three tokens force exponential rank.  The optimal length-three exponent
remains open.  For example, at $m=168$, the bound in
Appendix~\ref{app:amplified-rank} gives $r\ge32{,}768$ for correctness on all
Boolean inputs.

\paragraph{Dense comparison and the matching feature upper bound.}
Dense softmax instead uses scores $-\tau\ip{x_i}{x_j}$, values $v(x_j)=x_j$,
and readout $\ip{x_i}{a_i(X)}$.  For $0<\varepsilon\le1/2$, temperature
$\tau\ge\log(n/\varepsilon)$ ensures
$0\le\widehat t_i-t_i<\varepsilon$ at every length up to $n$.
At length three, $\tau=\log6$ suffices for error $<1/2$.
These are $m$-dimensional scores, but their exponentiated Boolean-domain
kernel has rank exactly $2^m$ for every $\tau>0$.  It has an exact
nonnegative feature factorization of that dimension, giving the same
prediction with $r=2^m$.  Appendix~\ref{app:upper-bounds} proves both claims.
The lower bound concerns every successful nonnegative kernel, including
kernels chosen without reference to softmax.

\subsection{Why the third token changes the problem}

Fix a query $x$.  Normalization lets us include the affine readout in the
scalar token values $g_x(z)=\beta(x)+w(x)^\top v(z)$.  For any input $X$,
write its prediction as $N_X/D_X$, where
\[
 N_X=\sum_{z\in X}\alpha(x,z)g_x(z),\qquad
 D_X=\sum_{z\in X}\alpha(x,z)>0.
\]
Both sums count repeated tokens.  The scalars $g_x(z)$ are unrestricted.
For example, choose $\ip{x}{y}=0$ and $\ip{x}{z}=3$.  Consider the following
three inputs:
\begin{equation}
\begin{array}{c@{\qquad}c@{\qquad}c}
 \text{input}&\text{tokens}&\text{answer at }x\\[2pt]
 Y&\witnessbox{x}\thickspace\witnessbox{y}\thickspace\witnessbox{y}&0\\[3pt]
 Z&\witnessbox{x}\thickspace\witnessbox{z}\thickspace\witnessbox{z}&3\\[3pt]
 M&\witnessbox{x}\thickspace\witnessbox{y}\thickspace\witnessbox{z}&0
\end{array}
\label{eq:witness-inputs}
\end{equation}
The mixed input must keep the low answer when one copy of $y$ is replaced
by $z$.  Yet every token contributes linearly to the numerator, so
\begin{equation}
 N_Y+N_Z=2N_M.
 \label{eq:witness-identity}
\end{equation}
This identity cancels each unknown value, including the query's own
contribution.  To achieve error $<1/2$ on all three inputs, normalization
must give the better key more than twice the weight of the worse key:
$\alpha(x,y)>2\alpha(x,z)$.  This step requires nonnegative kernel entries.
Values and readout coefficients may have either sign.
Appendix~\ref{app:gap-domination} derives this inequality and its general
fixed-length form.  The first experiment in Section~\ref{sec:experiments} uses
these same three witness types.

To make the ratio larger, insert candidate
keys with successively larger overlaps with $x$.  Every gap of three forces
another factor of two, so a chain of $k$ comparisons forces a factor
$2^k$.  The Boolean domain allows a number of steps proportional to $m$.
Each step uses only three-token inputs, so the whole argument
stays at length three.

Finally, choose many well-separated queries with the same number of ones,
forming a constant-weight code.  Pair each with
its complement as the preferred key.  Code distance makes every other
paired key much worse.  The resulting kernel submatrix, after normalizing
its diagonal to one, has small off-diagonal entries.  The trace--Frobenius
inequality forces such a matrix to have large rank, even without symmetry.
Because any $r$-feature kernel has matrix rank at most $r$, this proves the
separation.  Appendix~\ref{app:amplified-rank} gives the complete finite bound.

\subsection{Sharper bounds and other input settings}

The constant $c_3$ comes from using overlap gaps of five, each forcing a
factor of four, and using the distance of each codeword pair.
Counting pairs at each distance gives a stronger Frobenius bound than
using only the minimum distance.  Appendix~\ref{app:shell} gives the finite
bound with its polynomial prefactor.

\paragraph{Growing context.}
Let the integer sequence length $n=n(m)$ tend to infinity.  Correctness on
every sequence of that exact length with error $<1/2$ requires $r\ge2^{m-o(m)}$.
Longer contexts strengthen the preference at each step, approaching the
$2^m$ upper bound in the exponent.  Appendix~\ref{app:growing-context}
proves this with bounds uniform in $n$.

\paragraph{Positions, causality, and distinct candidates.}
The preceding one-head lower bounds extend to position-dependent tokenwise
maps, with correctness required at a fixed query position.  The model still
has one layer, an affine readout, positive denominators, and no other
cross-token channel.  At each visible source position $j$, assume a
nonnegative kernel $A_j(x,z)=\ip{\phi_Q(x)}{\phi_{K,j}(z)}$.
The bounds apply to the shared query feature dimension $r$.  If each position
has an independent query feature space of dimension $r_j$, they apply to
$\sum_jr_j$ instead.  They also hold for a final causal query attending to
all preceding tokens and itself, under the same length and error assumptions
(Appendix~\ref{app:position-extension}).

An exponential lower bound also holds when correctness with error $<1/2$
is required only on inputs of exactly three distinct tokens with a unique
minimizing token at the queried position:
$r\ge2^{(c_{\mathrm{uniq}}-o(1))m}$, where
$c_{\mathrm{uniq}}=1-\log_2(1+3^{-1/7})\approx0.11$.
Other positions may have ties.  This bound also permits the positional and
final-causal settings above
(Appendix~\ref{app:unique}).

\section{Multiple heads}
\label{sec:multihead}

Several heads can distribute the comparison work, and signed output
coefficients can cancel their predictions.  We therefore cannot simply
apply the one-head argument to each head separately.
For $H$ normalized nonnegative heads, with feature dimensions $r_h$ and
arbitrary finite value dimensions $d_{v,h}$, affine mixing gives
\begin{equation}
 F_x(X)=\beta(x)+\sum_{h=1}^H
 \frac{\sum_j\alpha_h(x,x_j)w_h(x)^\top v_h(x_j)}
      {\sum_j\alpha_h(x,x_j)}.
 \label{eq:multihead-model}
\end{equation}
Every head denominator is positive.  Maps are tokenwise and
position-independent, arithmetic is exact, and there is one attention
layer with no other cross-token channel.

\begin{theorem}[Any fixed number of heads]
\label{thm:multihead}
For every fixed $H\ge1$, the model in Equation~\eqref{eq:multihead-model},
correct with error $<1/2$ on every sequence of exact length $H+2$, requires
\[
 R:=\sum_{h=1}^H r_h\ge2^{\Omega_H(m)}.
\]
Values and affine mixing may have either sign.  Each head may use different maps.
\end{theorem}

Averaging $H$ copies of the $2^m$-feature construction gives $R=H2^m$.
The explicit exponent in our lower bound decreases rapidly with $H$.
The sharp context-length threshold for $H>1$ remains open.

The key step finds a head that prefers a given better key to \emph{every}
sufficiently worse key.  The head can depend on the query and better key,
but is chosen before the worse key.  This stronger guarantee lets the proof
connect comparisons even when the head changes along the chain.

To see why such a head must exist, suppose each head has a worse key that
violates this preference.  Put these $H$ keys into $H+1$ source slots, repeating one,
and replace any subset of the slots by the better key.  Multiplying the
prediction by all head denominators gives a polynomial of degree at most
$H$ in the replacement indicators.  An alternating sum over $H+1$
independent replacements cancels every term: each term omits at least one
indicator.  Only the input with no replacement has a high target.
Nonnegative kernel entries bound how much each replacement can increase
the denominators.  For a large enough overlap gap, correctness then
contradicts the cancellation identity.

For example, with $H=2$ the base input is $(x,z_1,z_2,z_1)$.
Let $b\in\{0,1\}^3$ indicate which source positions are replaced by $y$,
giving $X_b$, and set $t=\ip{x}{y}$.  Write each head's contribution as
$N_h(b)/D_h(b)$, including the self token.  Both numerator and denominator
are affine in $b$, so
\[
 \begin{aligned}
 P(b)&=(F_x(X_b)-t)D_1(b)D_2(b)\\
     &=(\beta(x)-t)D_1(b)D_2(b)+N_1(b)D_2(b)+N_2(b)D_1(b)
 \end{aligned}
\]
has degree at most two.  Thus
$\sum_{b\in\{0,1\}^3}(-1)^{b_1+b_2+b_3}P(b)=0$:
every term cancels when summing over a coordinate it omits.
The target is high only at $b=0$ and equals $t$ for every other choice.
Three replaceable source positions supply the cancellation, giving four
tokens in total when the query is included.

Along the resulting chain, some head appears often enough to amplify its
preference exponentially, even if it appears at nonconsecutive steps.
Different queries may select different heads.  Selecting their rows from
the head kernel matrices gives $B=\sum_hR_hA_h$, where $R_h$ selects rows
assigned to head $h$, and hence $\operatorname{rank}B\le\sum_hr_h$.
The same near-identity argument completes the proof
(Appendix~\ref{app:multihead}).

The exponential total-rank conclusion also survives arbitrary
position-dependent tokenwise maps at a fixed query position, including a final causal
query over $H+1$ sources.  Each head's shared query space has dimension $r_h$.
Independent position spaces are counted by their total dimension.
The proof uses products of subsets of head kernels and counts their
product feature dimensions (Appendix~\ref{app:positions}).

\section{Decoder degree and finite precision}
\label{sec:multilevel}

The rank bounds already allow exact arithmetic.  We now vary what the
query can compute from the head output.  Allowing polynomial readouts changes
the context threshold.  Unrestricted decoding can solve Min-IP even at rank one.

\subsection{A hierarchy in context length and decoder degree}

A polynomial readout combines sums and products of head-output coordinates.
Its degree is the largest number of factors in a product, counting repetitions.
Degree one gives an affine readout.  Degree two also allows products of two
coordinates, such as a square.

\begin{theorem}[One-head degree/context hierarchy]
\label{thm:polynomial}
Fix an integer $D\ge1$, independent of $m$.  In the position-independent one-head model of
Section~\ref{sec:kernel-model}, replace affine readout by an arbitrary
query-dependent polynomial of total degree at most $D$ in the head output.
Keep nonnegative kernels, positive denominators, arbitrary finite-dimensional
tokenwise values, exact arithmetic, and no other cross-token channel.  Then:
\[
\begin{array}{ll}
 n=D+1:&\text{rank one suffices for exact Min-IP},\\[2pt]
 n=D+2:&\text{error }<1/2\text{ on every input requires }r\ge2^{\Omega_D(m)}.
\end{array}
\]
Both statements concern exact length $n$.  The upper construction uses
$2^m$ value coordinates and may depend on $n$.
\end{theorem}

For $D=1$, this recovers the two-token/three-token threshold.  For $D=2$,
rank one suffices for three tokens and exponential rank is necessary for
four.  The construction's $2^m$ value coordinates give exponential total
width.  Its decoder is tailored to the exact length, with no guarantee at
shorter lengths.  Fixed polynomial degree does not cover arbitrary ReLU networks.

Uniform attention with one-hot values records token frequencies.
Multiplying by the fixed length $D+1$ gives counts, each linear in the head
output.  For $k=1,\ldots,\|x\|_1$, let $s_k$ count tokens with overlap at
least $k$.  The query qualifies, so $s_k\in\{1,\ldots,D+1\}$.
The degree-$D$ polynomial
$J(s)=(s-1)\cdots(s-D)/D!$ is zero at $1,\ldots,D$ and one at $D+1$.
Thus $\sum_kJ(s_k)$ is a readout of degree at most $D$ that counts thresholds
met by every token, giving the minimum overlap.

For the lower bound, value numerators and the head denominator are affine
in the replacement indicators.  Multiplying the prediction by the
denominator's $D$th power gives a polynomial of degree at most $D$.
At length $D+2$, every monomial therefore omits one of the $D+1$ source
indicators and cancels in the alternating sum.  As in the multihead proof,
correctness and nonnegative kernels force a uniform preference for
lower-overlap tokens when the gap is large enough.  Amplification gives
exponential rank.  Appendix~\ref{app:polynomial} proves both bounds.

\subsection{Unrestricted decoding changes the resource to count}

Uniform attention can retain all information needed for Min-IP when decoding
is unrestricted.  For each $m$, a position-independent positive rank-one head with scalar
values and an unrestricted query-dependent exact-real decoder solves Min-IP
at every finite length, using maps independent of length.
Powers of a transcendental number encode the empirical distribution in one
scalar.  The decoder is discontinuous (Appendix~\ref{app:uniform-real}).
There is also a continuous construction: at a final query with two
position-identified sources, rank one and $2m$ value coordinates suffice
with one query-dependent ReLU readout (Appendix~\ref{app:relu-escape}).

For multiple layers with arbitrary local computation, we therefore count
finite-alphabet communication.  Token $i$ starts at $h_i^{(0)}=E(x_i,i)$.
At layer $\ell$, a deterministic map sends the same message
$\sigma_\ell\in\Sigma_\ell$ to every token, using all current token states
to compute it.  Each token then updates from its state, position, and
$\sigma_\ell$.  The final readout is tokenwise.  The list of messages
$(\sigma_1,\ldots,\sigma_L)$ is the \emph{transcript}.
Every cross-token channel must appear in this transcript.
Local state and computation are unrestricted.

Our hard family contains $q$ fixed query tokens and $q$ database tokens.
Each database token independently selects its matching query's answer from
$\{0,\ldots,s\}$.  For integers $M,s\ge1$, take $m=2+M+s$ and a code
$\cC\subseteq\binom{[M]}h$, with $1\le h\le M$ and directed distance at
least $s+1$, meaning $|S\setminus T|\ge s+1$ for distinct codewords $S,T$.
Choose any $1\le q\le|\cC|$ codewords.  Code distance
separates nonmatching database tokens, and type coordinates separate query
tokens, so every unintended overlap at a queried token exceeds $s$.
The construction and its overlap checks are in
Appendix~\ref{app:information}.

\begin{theorem}[Transcript injectivity]
\label{thm:transcript}
For this family with $2q\le n$, query tokens and positions are fixed across
assignments.  A deterministic sketch model with finite nonempty alphabets
$\Sigma_1,\ldots,\Sigma_L$ and error $<1/2$ on all inputs of length up to
$n$ must satisfy $\sum_{\ell=1}^L\log_2|\Sigma_\ell|\ge q\log_2(s+1)$,
provided every cross-token channel is counted.
\end{theorem}

Two assignments with the same transcript give every fixed query the same
output.  Strict error $<1/2$ identifies its integer answer, so all
$(s+1)^q$ assignments need different transcripts.  In particular, if each of
$H$ heads in each of $L$ layers communicates only
$S_{\ell,h}\in\mathbb R^{r\times d_v}$ and $z_{\ell,h}\in\mathbb R^r$,
with at most $2^p$ values per coordinate, then
\begin{equation}
 LHr(d_v+1)p\ge q\log_2(s+1).
 \label{eq:finite-precision}
\end{equation}
Any additional cross-token channel adds to this count.
For $m\ge32$, setting $s=\lfloor m/\log_2m\rfloor$ permits
$|\cC|\ge2^{m-O(m\log\log m/\log m)}$
(Appendix~\ref{app:multilevel-code}).  For $n\ge2$, choosing
$q=\min\{\lfloor n/2\rfloor,|\cC|\}$ gives $\Omega(q\log m)$ bits,
and $\Omega(n\log m)$ when $\lfloor n/2\rfloor\le|\cC|$.

\section{Experiments}
\label{sec:experiments}

The experiments illustrate the proof mechanism and study learned behavior
at finite scales: the proof's length-three family, generalization across task
scales, and head allocation at fixed budgets.  All experiments use labels
checked by an exact integer oracle and five fixed training seeds.
The code supplies configurations, diagnostics,
and reproduction commands.

\subsection{Finite three-token family}

The construction starts with fifteen 24-bit vectors, each containing 12 ones,
with directed distance at least six.  They define key chains with overlap jumps
of at least three.  Each jump gives the proof's
witnesses $(x,y,y)$, $(x,z,z)$, and $(x,y,z)$.  This gives 424 jumps and $1{,}272$ inputs.
Strict correctness forces $\alpha(x,y)>2\alpha(x,z)$ at each jump, amplifying
ratios along the chain (Figure~\ref{fig:three-token-witness}(c)).

We train models at six ranks using positive query/key tables, full one-hot tokenwise
values, and arbitrary query-dependent affine readouts.  Positive factors
form a smaller class than the nonnegative kernels allowed by the theorem.
The strict certificate excludes $r<9$ (Appendix~\ref{app:finite-certificate}).
Mean maximum error exceeds $0.80$ through $r=15$.  All five $r=32$ runs solve
the family with error below $1/2$.  A constructed $r=15$ solution shows that
training failure at that rank comes from optimization.

\begin{figure}[H]
  \centering
  \includegraphics[width=0.94\textwidth]{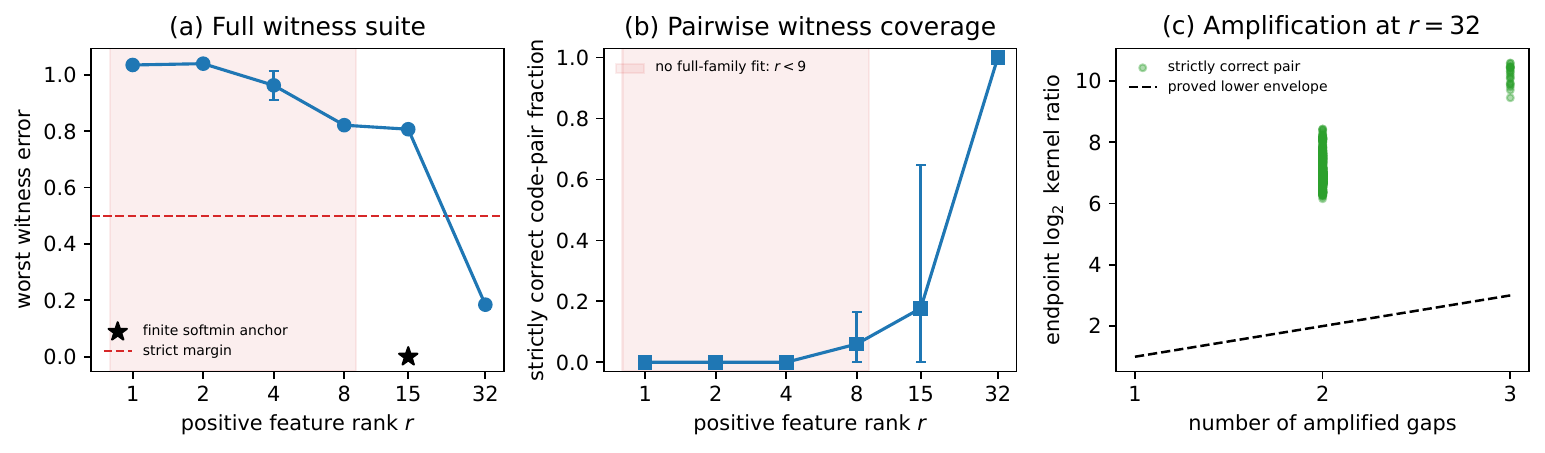}
  \caption{Feature-rank sweep: maximum error, ordered-pair coverage, and
  $r=32$ endpoint ratios against the dashed bound of a factor of two per step.
  Shading marks proof-excluded ranks, the star marks the constructed $r=15$ model,
  and bars are 95\% Student-$t$ intervals.}
  \label{fig:three-token-witness}
\end{figure}

\subsection{Feature capacity on held-out inputs}

The model must identify which of $q$ queries have a matching database token.
One database slot per query holds its zero-overlap match (target zero) or
padding with overlap one (target one).  Tokens built from $M$-bit codewords
are shuffled, with fixed length $n=2q$ and no positional features.
The joint out-of-distribution (OOD) test holds out patterns of
present matches, base codewords, and coordinate permutations.
All Boolean token vectors in the joint test also occur in training, but their
complete sequences are absent from training and validation, even ignoring order.

We vary rank at $(M,q,n)=(6,4,8),(8,6,12),(10,8,16)$, training one-head
positive-feature models at nine ranks for 600 steps.  We also have a dense
reference.  With $M+2\le12$ input coordinates,
$d_v=16$ spans the full affine linear value/readout class in these models.
The theorem allows more general tokenwise values.  Exact-sequence accuracy
counts an input as correct only if all $q$ query errors are below $1/2$.
The smallest tested rank reaching $90\%$ mean accuracy increases from
$r=8$ to $16$ to $64$ (Figure~\ref{fig:rank-phase}).  Dense attention succeeds
at every seed.  Since $M,q,n$ grow together, the comparison measures their
combined effect.

\begin{figure}[H]
  \centering
  \includegraphics[width=\textwidth]{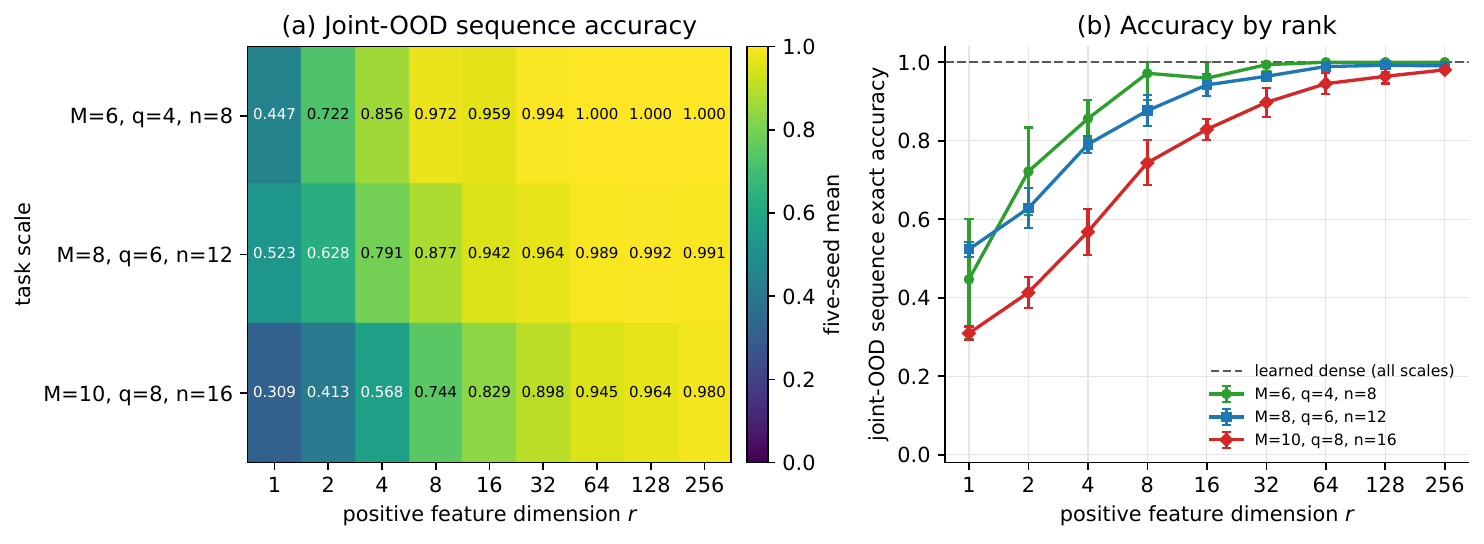}
  \caption{Joint-OOD exact-sequence accuracy with value dimension $d_v=16$:
  five-seed means in (a) and 95\% Student-$t$ intervals in (b).
  Dashed lines show the dense reference.}
  \label{fig:rank-phase}
\end{figure}

\subsection{Fixed-budget head allocation}
\label{sec:head-allocation}

We fix $(M,q,n)=(8,3,6)$ and divide $R=\sum_h r_h$ features and total value
width 64 equally among $H=1,2,4$ heads.  Each head has enough value coordinates
to represent the full linear value class.  Positive models match parameter counts, initial parameter values,
and minibatch order at each $R$.  All models, including a separate dense
control, receive 1,200 steps with validation-only checkpoint selection.
Figure~\ref{fig:head-allocation} compares two data draws.  One head has the
lowest mean error, but the two- and four-head curves cross.  On the second draw,
all paired 95\% intervals for the $H=2$ versus $H=1$ error difference include
zero.  At larger budgets, models fit the training data almost perfectly but
still make errors on held-out inputs.

\begin{figure}[H]
  \centering
  \includegraphics[width=\textwidth]{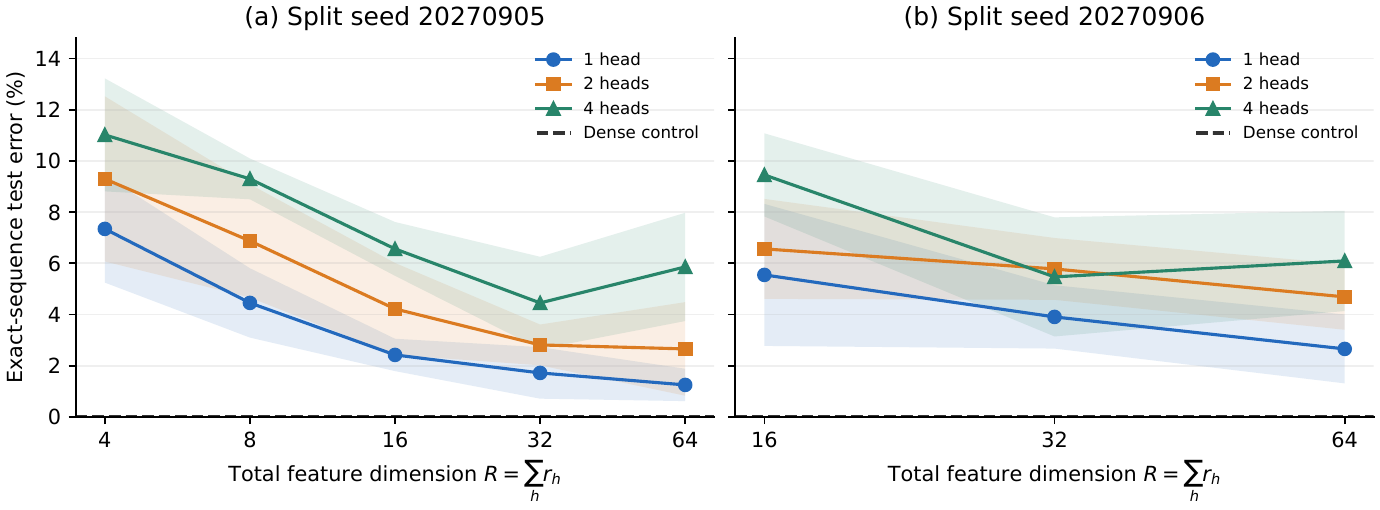}
  \caption{Head allocation at length six: mean exact-sequence test error with
  95\% Student-$t$ bands over five paired training seeds per dataset draw.
  The dashed dense control has zero error.}
  \label{fig:head-allocation}
\end{figure}

\section{Related work}

Inner-product tasks also underlie computational lower bounds for attention.
Big Bird bounds sparse-attention depth under the Orthogonal
Vectors assumption~\citep{zaheer2020bigbird}, and
\citet{alman2025fundamental} prove conditional runtime bounds for
document-similarity problems.  Our results bound nonnegative kernel feature
rank unconditionally.  Table~\ref{tab:related-resources} compares this resource
with those in the closest results on expressive power.

\begin{table}[ht]
  \centering
  \caption{Closest separations and the resources they bound.}
  \label{tab:related-resources}
  \begin{tabular}{@{}p{0.26\textwidth}p{0.30\textwidth}p{0.38\textwidth}@{}}
    \toprule
    Work & Resource & Main guarantee \\
    \midrule
    \citet{amsel2025quality}
      & Per-head projection width and heads
      & Nearest neighbor: two candidates and a query \\
    \citet{hsu2026attention}, Thm.~1
      & Value width, precision, range
      & One-head simultaneous min/max \\
    \citet{hsu2026attention}, Thm.~2
      & Heads $\times$ polynomial degree
      & Boolean threshold-degree lower bound \\
    This work
      & \raggedright Nonnegative kernel factorization rank
      & Worst-case Min-IP with affine heads or one-head polynomial decoding \\
    \bottomrule
  \end{tabular}
\end{table}

\citet{amsel2025quality} study nearest-neighbor approximation with two
candidates and a separate query.  Their lower bound trades off head count
against the width of linear candidate projections, measuring expected squared
error over spherical inputs.  We bound tokenwise kernel feature rank over
Boolean inputs, requiring strict accuracy on every input with arbitrary
values and query-dependent affine decoding.

\citet{hsu2026attention} lower-bound head count for polynomial threshold
decoding using Boolean threshold degree, with unrestricted value width and
precision.  Their proof clears attention denominators.  We use this step
to force kernel preferences and amplify them across comparisons, bounding
total feature rank.  We also give a matching exact-length upper bound for
one-head polynomial decoding.
Complementary work studies recovery of softmax heads from exact input-output
queries~\citep{kim2026provably} and how approximation depends on head count
and sequence length, including positional full-sequence encoding for nonlinear
decoding~\citep{yu2026effect}.

A related question asks whether set similarities admit single-vector
embeddings.  \citet{jayaram2026multivector} and \citet{jayaram2026maxip}
lower-bound the dimension needed for one inner product to approximate
Chamfer/MAX-IP similarities, even with data-dependent embeddings.  Their
bounds scale with the number of vectors per document.  Our three-token
theorem fixes two candidates and bounds tokenwise kernel rank as their
Boolean dimension grows.

At the kernel level, Performers approximate softmax with positive random
features~\citep{choromanski2021performers}, while KATA studies cone geometry
in nonnegative attention~\citep{ghriss2026kata}.  Our rank bounds allow
arbitrary nonnegative kernels, excluding signed methods such as
ZeroS~\citep{lu2026zerosum}.  Communication reductions yield state-size bounds
for finite-precision recurrent models in \citet{arora2024based} and
\citet{bhattamishra2024separations}.  We use a multilevel Min-IP family to
bound finite-alphabet transcripts across heads and layers.

\section{Limitations and conclusion}

Our rank bounds concern one nonnegative-kernel layer: a fixed number of heads
with query-dependent affine decoding, or one head with fixed-degree polynomial
decoding.  For arbitrary depth and local computation, the transcript bound
assumes finite-alphabet cross-token communication.  Turning these resource
bounds into running-time bounds requires a computational model.
The optimal length-three exponent and a sharp multihead length threshold
remain open.

The results show why an attention summary must be assessed together with its
decoder.  Under exact arithmetic, a scalar can retain everything needed for
Min-IP, yet one nonnegative head with an affine readout requires exponential
feature rank at length three.  The polynomial hierarchy further connects rank
to decoding, with exponential value width in its upper construction.  This
suggests studying short-context comparison to understand which answers a
restricted readout can extract from the information preserved by aggregation.
\label{page:main-end}

\bibliography{ref}
\bibliographystyle{plainnat}

\clearpage
\appendix
\section{Single-head rank bounds: complete proofs}
\label{app:proofs}

We first verify the upper constructions from Section~\ref{sec:rank}, then
derive the weight comparisons and rank estimates used in
Theorem~\ref{thm:three-token}.  The remaining arguments establish the
growing-context, positional, shell-counting, and distinct-token claims in
that section.  Logarithms in softmax temperatures are natural, and every
prediction bound uses the strict inequality in Equation~\eqref{eq:strict-error}.

\subsection{Dense softmax and exact exponential features}
\label{app:upper-bounds}

\begingroup
\subsubsection*{Dense softmax approximation}
\phantomsection\label{app:dense-softmin}
Let $m,n\ge1$ be integers and $0<\varepsilon\le1/2$.  Use one dense softmax
head with scores $s_{ij}=-\tau\ip{x_i}{x_j}$, values $v(x_j)=x_j\in\mathbb R^m$,
and query-dependent linear readout $\widehat t_i=\ip{x_i}{a_i(X)}$.
We prove that $\tau\ge\log(n/\varepsilon)$ gives
$0\le\widehat t_i-t_i<\varepsilon$ on every Boolean sequence of length
$1\le N\le n$, including every self token in the normalization and target.
Fix query $i$, put $t=t_i$, and let
$M_i=\{j:\ip{x_i}{x_j}=t\}$.  For
$d_j=\ip{x_i}{x_j}-t\in\mathbb Z_{\ge0}$, the readout averages the overlaps
$t+d_j$.  Cancelling the common factor $e^{-\tau t}$ gives
\begin{equation}
 \widehat t_i-t
 =\frac{\sum_jd_je^{-\tau d_j}}{\sum_je^{-\tau d_j}}
 =\frac{\sum_{j\notin M_i}d_je^{-\tau d_j}}
 {|M_i|+\sum_{j\notin M_i}e^{-\tau d_j}}.
 \label{eq:app-softmin-excess}
\end{equation}
The expression is nonnegative.  The assumption
$\tau\ge\log(n/\varepsilon)$ implies $\tau\ge\log2$.  For every integer
$d\ge1$, $de^{-\tau d}\le e^{-\tau}$ because
$d\le2^{d-1}$ and $e^{-\tau(d-1)}\le2^{-(d-1)}$.  There is at least one
minimizer, so $|M_i|\ge1$ and
\[
 0\le\widehat t_i-t
 \le\frac{(N-|M_i|)e^{-\tau}}{|M_i|}
 \le(n-1)e^{-\tau}\le\frac{n-1}{n}\varepsilon<\varepsilon.
\]
For $n=1$ Equation~\eqref{eq:app-softmin-excess} is zero.  The bound allows
repetitions, ties, and the self token.  All softmax weights are positive,
so the denominator is nonzero.  In particular, $\tau=\log6$ suffices at
length three for error strictly below $1/2$.
\endgroup

\begingroup
\subsubsection*{An exact feature factorization}
\phantomsection\label{app:exact-feature}
For every $m\ge1$ and $\tau>0$, we prove that the Boolean-domain kernel
$K_\tau(x,z)=e^{-\tau\ip{x}{z}}$ has an exact nonnegative feature
factorization of dimension $2^m$, and no real bilinear factorization of
smaller dimension, even if its coordinates may have either sign.
Index coordinates by $u\in\bits^m$ and define
\[
 \phi_Q(x)=e_x,\qquad \phi_K(z)_u=e^{-\tau\ip{u}{z}}.
\]
Here $e_x\in\mathbb{R}^{2^m}$ denotes the standard basis vector indexed by
$x$, with a one in coordinate $x$ and zeros elsewhere.
Both maps are nonnegative.  The query vector selects coordinate $x$ of
$\phi_K(z)$, so their inner product is $e^{-\tau\ip{x}{z}}$.
Every kernel entry is positive, so normalization is defined.
With values $v(z)=z$ and readout $\ip{x_i}{a_i(X)}$, this gives exactly
the dense construction.  Thus for each $n\ge1$ and
$0<\varepsilon\le1/2$, choosing $\tau\ge\log(n/\varepsilon)$ gives one
position-independent model correct with error $<\varepsilon$ on all
lengths up to $n$.  In particular, $r^\star(m,n)\le2^m$.

To prove minimality for this exact kernel, factor each entry over the $m$
bit coordinates.  Ordering Boolean words by their bits, the full kernel matrix is
\[
 \begin{pmatrix}1&1\\1&e^{-\tau}\end{pmatrix}^{\otimes m}.
\]
The one-bit determinant is $e^{-\tau}-1\ne0$ for $\tau>0$, so that matrix
has rank two.  Kronecker-product rank is multiplicative.
Hence the cube matrix has rank $2^m$, while any bilinear
factorization through $\mathbb R^r$ has matrix rank at most $r$.
\endgroup

\subsection{Fixed-length amplification and the phase transition}

The argument has three steps: cancel the unknown values to force a kernel
preference, amplify that preference along overlap levels, and turn it into
a rank bound using a constant-weight code.  After proving
Theorem~\ref{thm:three-token}, we apply the same estimates to the finite
witness family, growing contexts, and position-dependent maps.

\begingroup
\subsubsection*{Weight domination at a fixed length}
\phantomsection\label{app:gap-domination}
Use the one-head architecture of Section~\ref{sec:kernel-model}: a
normalized nonnegative kernel with an $r$-dimensional real factorization,
arbitrary finite-dimensional tokenwise values, and a query-dependent affine
readout.  Arithmetic is exact, maps are position-independent, there is one
layer and no other cross-token channel, and every denominator is positive.
Fix $0<\varepsilon\le1/2$ and an exact length $n\ge3$.  Suppose the model's
error is strictly below $\varepsilon$ on every sequence of that length.
For Boolean tokens $x,y,z$, let
\[
 t=\ip{x}{y}<u=\ip{x}{z}\le\ip{x}{x},\qquad d=u-t.
\]
For $d\ge2\varepsilon(n-1)/(n-2)$, we prove
\begin{equation}
 \alpha(x,y)>
 \frac{(n-2)(d-2\varepsilon)}{2\varepsilon}\alpha(x,z).
 \label{eq:app-gap-domination}
\end{equation}
For a fixed query $x$, we can combine the values and affine readout into
scalars $g_x(z)=\beta(x)+w(x)^\top v(z)$, because normalized weights sum to
one.  Thus
\[
 \widehat t_x(X)=
 \frac{\sum_j\alpha(x,x_j)g_x(x_j)}{\sum_j\alpha(x,x_j)}.
\]

Write $L=n-1$, $K=n-2=L-1$ and
$a=\alpha(x,x)$, $b=\alpha(x,y)$, $c=\alpha(x,z)$.  Subtract
$t=\ip{x}{y}$ from all scalar values and set
\[
 p=g_x(x)-t,\qquad q=g_x(y)-t,\qquad s=g_x(z)-t.
\]
Consider the three exact-length inputs
\[
Y=(x,y^L),\qquad Z=(x,z^L),\qquad M=(x,y,z^K).
\]
Here powers denote repeated copies of a token.  The condition
$t<u\le\ip{x}{x}$ makes their shifted query targets $0,d,0$, even with
the self token included.  Their numerators and denominators are
\[
\begin{array}{lll}
N_Y=ap+Lbq,&&D_Y=a+Lb,\\
N_Z=ap+Lcs,&&D_Z=a+Lc,\\
N_M=ap+bq+Kcs,&&D_M=a+b+Kc.
\end{array}
\]
Positive denominators and strict correctness give
\[
 N_Y>-\varepsilon D_Y,\qquad
 N_Z>(d-\varepsilon)D_Z,\qquad
 N_M<\varepsilon D_M.
\]
The exact identity
$N_M-L^{-1}N_Y-(K/L)N_Z=0$ cancels $ap,bq,cs$ separately.  Combining the
upper bound for $N_M$ with the lower bounds for the two subtracted terms gives
\[
 0<\varepsilon D_M+\frac{\varepsilon}{L}D_Y
      -\frac{K(d-\varepsilon)}{L}D_Z.
\]
Expanding the denominators gives
\[
0<\left(2\varepsilon-\frac{Kd}{L}\right)a
  +2\varepsilon b-K(d-2\varepsilon)c.
\]
The gap assumption makes the coefficient of $a$ nonpositive.  Since $a\ge0$,
we can drop this term to obtain
\[
 b>\frac{(n-2)(d-2\varepsilon)}{2\varepsilon}c.
\]
Every inequality remains strict, repeated tokens are legal, and no other
sequence length is used.  Values and readout coefficients may have either
sign.  For $\varepsilon=1/2$, any integer gap $d\ge2$
gives the factor $(n-2)(d-1)$.
\endgroup

\begingroup
\subsubsection*{A trace--Frobenius rank estimate}
\phantomsection\label{app:approx-identity}
The matrix estimate needed after amplification does not require symmetry.
Let $t\ge1$ be an integer and $\lambda>0$.  If
$C\in\mathbb R^{t\times t}$ has $C_{ii}=1$ and
$|C_{ij}|\le1/\lambda$ for $i\ne j$, we prove
\[
 \operatorname{rank}(C)\ge\frac{t}{1+(t-1)/\lambda^2}.
\]
The assumptions imply
\[
 \operatorname{tr}C=t,\qquad
 \|C\|_F^2\le t+\frac{t(t-1)}{\lambda^2}.
\]
Let $k=\operatorname{rank}(C)$ and write a singular value decomposition
$C=\sum_{j=1}^k\sigma_j u_jv_j^\top$, with unit vectors $u_j,v_j$.
Since $|v_j^\top u_j|\le1$, Cauchy--Schwarz gives
\[
 |\operatorname{tr}C|
 =\left|\sum_{j=1}^k\sigma_jv_j^\top u_j\right|
 \le\sum_{j=1}^k\sigma_j
 \le\sqrt{k}\left(\sum_{j=1}^k\sigma_j^2\right)^{1/2}
 =\sqrt{k}\,\|C\|_F.
\]
Squaring and substituting the trace and Frobenius bounds proves the claim.
\endgroup

\begingroup
\subsubsection*{Amplification across overlap levels}
\phantomsection\label{app:amplified-rank}
Assume the architecture and correctness at an exact length from
Appendix~\ref{app:gap-domination}.  Fix an integer $1\le h\le m$ and a
nonempty constant-weight family $\cG\subseteq\binom{[m]}h$ with directed
distance at least $\Delta$, meaning $|S\setminus U|\ge\Delta$ for distinct
$S,U\in\cG$.  Choose an integer $g\le\Delta$ such that
$g>2\varepsilon(n-1)/(n-2)$, and put
\[
 \mu_g=\frac{(n-2)(g-2\varepsilon)}{2\varepsilon}>1,
 \qquad \Lambda=\mu_g^{\lfloor\Delta/g\rfloor}.
\]
We prove
\begin{equation}
 r\ge\frac{|\cG|}{1+(|\cG|-1)/\Lambda^2}
 \ge\frac12\min\{|\cG|,\Lambda^2\}.
 \label{eq:amplified-rank}
\end{equation}
If $|\cG|=1$, the first displayed bound reduces to $r\ge1$.  Indeed, $r=0$
makes every kernel weight and hence every denominator zero.  Now assume
$|\cG|\ge2$.  For $S,U\in\cG$, set
$B_{S,U}=A_{S,U^c}=\alpha(x_S,x_{U^c})$.
We first show $B_{S,S}>\Lambda B_{S,U}$ whenever $S\ne U$.

Fix distinct $S,U\in\cG$.  In row $S$ of $B$, take the query $x=x_S$ and the
two key tokens $y_0=x_{S^c}$ and $y_k=x_{U^c}$.  Then
$\ip{x}{y_0}=0$ and
$\ip{x}{y_k}=|S\setminus U|=:D\ge\Delta\ge g$, while
$D\le|S|=h=\ip{x}{x}$.  To amplify this comparison, connect the two tokens
through intermediate overlap levels.
Write
\[
 k=\left\lfloor\frac Dg\right\rfloor,
 \qquad D=kg+r_0,
 \qquad0\le r_0<g.
\]
Set $\ell_0=0$ and $\ell_i=ig+r_0$ for $1\le i\le k$.  For
$1\le i<k$, let $y_i$ be the indicator of any $\ell_i$ coordinates in
$\operatorname{supp}(x)$.  Such a choice exists because
$\ell_i<D\le\|x\|_1$.  It satisfies $\ip{x}{y_i}=\ell_i$, and the endpoints have
overlaps $\ell_0$ and $\ell_k=D$.  The first gap is $g+r_0$ and every later gap is
$g$.  Thus we include the remainder $r_0$ in the first step, so no gap is
smaller than $g$.  Equation~\eqref{eq:app-gap-domination} therefore gives a
factor of at least $\mu_g$ at each of the $k$ steps, so
\begin{equation}
 \alpha(x,y_0)>
 \mu_g^{\lfloor D/g\rfloor}\alpha(x,y_k).
 \label{eq:app-chain}
\end{equation}

Equation~\eqref{eq:app-chain} and
$\lfloor D/g\rfloor\ge\lfloor\Delta/g\rfloor$ give
$B_{S,S}>\Lambda B_{S,U}$.  Since $|\cG|\ge2$ and $B$ is nonnegative, every
diagonal entry is positive.
Let $R=\operatorname{diag}(B_{S,S})$ and normalize each row by setting $C=R^{-1}B$.
Then $C_{S,S}=1$ and $0\le C_{S,U}<1/\Lambda$.  This invertible row scaling
preserves rank.  Also, $B$ is a submatrix of the kernel matrix, which need not
be symmetric, and $\operatorname{rank}(A)\le r$.
The approximate-identity bound in Appendix~\ref{app:approx-identity} gives
the first inequality.  Finally, writing $t=|\cG|$,
\[
\frac{t}{1+(t-1)/\Lambda^2}
\ge\frac12\min\{t,\Lambda^2\}
\]
because the denominator is at most $2$ when $t\le\Lambda^2$ and at most
$2t/\Lambda^2$ when $t>\Lambda^2$.
\endgroup

\begingroup
\subsubsection*{Packing the middle layer}
\phantomsection\label{app:packing}
For even $M$ and integer $1\le\Delta\le M/2$, greedy selection from
$\mathcal L=\binom{[M]}{M/2}$ gives a directed-distance-$\Delta$ code of size
\begin{equation}
 T\ge
 \frac{\binom{M}{M/2}}
 {\sum_{a=0}^{\Delta-1}\binom{M/2}{a}^2},
 \label{eq:app-greedy-packing}
\end{equation}
Indeed, fix $S\in\mathcal L$.  A word $U\in\mathcal L$ at directed distance
$a$ is obtained by removing $a$ elements of $S$ and adding $a$ elements from
$S^c$, giving exactly $\binom{M/2}{a}^2$ choices.  Equal word sizes give
$|S\setminus U|=|U\setminus S|$.  Select one word and delete all words at
distance below $\Delta$, then repeat.  Each selection deletes at most the
denominator in~\eqref{eq:app-greedy-packing}, while the selected words have
the required directed distance.  When the process stops, every word has
been deleted.  The number of selections times the maximum number deleted
per selection must therefore be at least $|\mathcal L|$, proving the bound.
\endgroup

\begin{proof}[Proof of Theorem~\ref{thm:three-token}]
We combine the upper construction and rank estimates above.

\paragraph{Rank one at lengths one and two.}
Feature dimension zero makes every denominator zero, so $r^\star\ge1$.
For length at most two, use $\alpha(x,z)=1$, $v(z)=z$, and
\[
 \widehat t_i=2\ip{x_i}{a_i(X)}-\|x_i\|_1.
\]
At length one this is the self-overlap.  At length two it is the overlap with
the other token, which cannot exceed the self-overlap and hence is the target.

\paragraph{Exponential rank at length three.}
For the lower bound, take an even active dimension $M\ge168$ and
$\Delta=\lfloor M/8\rfloor$.  Write
$H_2(p)=-p\log_2p-(1-p)\log_2(1-p)$ for binary entropy.
Equation~\eqref{eq:app-greedy-packing} and the entropy bound give
\[
 \log_2T
 \ge(1-H_2(1/4))M-2\log_2(M+1)
 \ge M/10.
\]
The central coefficient is the largest of the $M+1$ binomial coefficients
whose sum is $2^M$, so $\binom{M}{M/2}\ge2^M/(M+1)$.
For $0<a<\Delta$, put $q=M/2$ and $p=a/q<1/4$.  A single term in the
expansion of $(p+(1-p))^q=1$ gives
\[
 \binom qa p^a(1-p)^{q-a}\le1,
 \qquad \binom qa\le2^{qH_2(p)}.
\]
Since $H_2$ increases on $[0,1/2]$, this implies
$\binom{M/2}{a}^2\le2^{M H_2(1/4)}$.
The term at $a=0$ is one and satisfies the same bound.  Summing at most
$M+1$ terms proves the first bound on $\log_2T$.

It remains to verify the last inequality for every $M\ge168$.  Define the
difference between its two sides by
\[
 F(M)=\bigl(9/10-H_2(1/4)\bigr)M-2\log_2(M+1).
\]
We check that $F$ is positive at $168$ and increases thereafter.
Since $H_2(1/4)=2-\tfrac34\log_2 3$, the endpoint calculation is
\[
 5F(168)=\log_2\!\left(\frac{3^{630}}{2^{924}\cdot169^{10}}\right)>0.
\]
The numerator exceeds the denominator by direct integer comparison.
For larger $M$, differentiation gives
\[
 F'(M)=\frac9{10}-H_2(1/4)-\frac{2}{(M+1)\ln2}
 >\frac7{80}-\frac3{169}>0.
\]
Here we used $\ln2>2/3$ and $H_2(1/4)<13/16$, the latter following from
$3^{12}>2^{19}$.  Thus $F(M)>0$ throughout the required range, proving
$T\ge2^{M/10}$.

Now apply Equation~\eqref{eq:amplified-rank} to this code with $h=M/2$, $n=3$,
$\varepsilon=1/2$, and $g=5$.  Then $\mu_g=4$ and
\[
 \Lambda^2=2^{4\lfloor M/40\rfloor}\le2^{M/10}\le T.
\]
Thus the minimum in Equation~\eqref{eq:amplified-rank} is $\Lambda^2$, and it
gives $r\ge\Lambda^2/2=2^{4\lfloor M/40\rfloor-1}$.  For odd $m$, fix the
final coordinate to zero and use $M=2\lfloor m/2\rfloor$, which preserves all
overlaps.  Finally, $M\ge m-1$ and
$\lfloor M/40\rfloor\ge M/40-1$ give
$4\lfloor M/40\rfloor-1\ge m/10-6$.
The sharper asymptotic claim follows from the distance-sensitive bound
in Appendix~\ref{app:shell} with $g=5$.
The exact-feature construction in Appendix~\ref{app:exact-feature} gives
the upper bound $r^\star(m,3)\le2^m$, completing the phase transition.
\end{proof}

\begingroup
\subsubsection*{The rank-nine witness certificate}
\phantomsection\label{app:finite-certificate}
The strict rank-nine certificate in Section~\ref{sec:experiments} follows
from the same argument, keeping the Frobenius bound strict.
Use the fifteen weight-twelve codewords in dimension 24, with minimum
directed distance six.  For every ordered pair, the witness family contains
all three length-three witness types at every jump of a gap-three chain.
Under the one-head architecture of Section~\ref{sec:kernel-model}, strict error
below $1/2$ on the whole family implies a strict factor-two domination at
each jump by Equation~\eqref{eq:app-gap-domination}.  There are at least two jumps,
so the kernel matrix $B_{S,U}=\alpha(x_S,x_{U^c})$ satisfies
$B_{S,S}>4B_{S,U}\ge0$ for $S\ne U$.  Every diagonal is therefore positive.
Row normalization gives $C_{S,S}=1$, $0\le C_{S,U}<1/4$ off diagonal,
and $\operatorname{rank}(C)=\operatorname{rank}(B)\le r$.  Hence
\[
 \|C\|_F^2<15+\frac{15\cdot14}{16},
 \qquad
 15^2\le\operatorname{rank}(C)\,\|C\|_F^2,
\]
where the second inequality follows from the nonsymmetric trace--Frobenius
argument above.  Thus
\[
 r\ge\operatorname{rank}(C)>\frac{15}{1+14/16}=8,
 \qquad\text{so }r\ge9.
\]
\endgroup

\begingroup
\subsubsection*{Growing exact context}
\phantomsection\label{app:growing-context}
Let the exact length $n=n(m)$ be integer-valued and tend to infinity.  Suppose
the one-head model of Section~\ref{sec:kernel-model} is correct on every
sequence of that length with error $<1/2$.  We prove $r\ge2^{m-o(m)}$,
allowing a separate choice of maps for each $m,n(m)$.
It suffices to consider $m\ge16$ and $n\ge258$.  Put
\[
 M=2\lfloor m/2\rfloor,
 \quad s=\log_2(n-2),
 \quad\Delta=2\left\lceil\frac{M}{2s}\right\rceil,
 \quad\rho=\frac{2\Delta}{M}.
\]
The distance is even and at least two.  Since $s\ge8$ and $M\ge16$,
$\Delta\le M/s+2\le M/4$.  This choice makes $\Lambda^2$ at least the
number of active Boolean words.  Indeed, with $g=2$ and
$\varepsilon=1/2$, we have $\mu_2=n-2$ and
\[
 \log_2\Lambda^2
 =2s\lfloor\Delta/2\rfloor=s\Delta\ge M.
\]
Each of the $\Delta$ shells deleted by greedy packing has size at most
$2^{M H_2(\rho)}$, by the entropy estimate above and $\rho\le1/2$.
The resulting code therefore satisfies
\[
 T\ge\frac{2^{M-MH_2(\rho)}}{(M+1)\Delta}.
\]
Since $T\le2^M\le\Lambda^2$, Equation~\eqref{eq:amplified-rank} gives
\[
 r\ge2^{M-MH_2(\rho)-\log_2((M+1)\Delta)-1}.
\]
When $n(m)\to\infty$, $\rho\le2/s+4/M=o(1)$, so $H_2(\rho)=o(1)$ and the
remaining logarithmic term is $O(\log m)=o(m)$ because $\Delta\le M/4$.
Since $M\ge m-1$, this is
$2^{m-o(m)}$.
\endgroup

\begingroup
\subsubsection*{Position-dependent maps and a final causal query}
\phantomsection\label{app:position-extension}
Fix one query position and $n-1$ visible source positions at an exact length
$n\ge3$.  Allow arbitrary position-dependent tokenwise query, key, and value
maps, with arbitrary finite value width and a query-dependent affine readout
at the fixed query position.  Suppose each source contribution
$A_j(x,z)=\ip{\phi_Q(x)}{\phi_{K,j}(z)}$ is nonnegative in a shared
$r$-dimensional query feature space.  Include the nonnegative self
contribution in the normalization.  Require positive denominators on all
declared inputs, exact arithmetic, one head, one layer, and no other
cross-token channel.  We prove all the one-head rank lower bounds above and
the distance-sensitive bound below under the same exact-length and
strict-error assumptions.

For the domination step, fix $0<\varepsilon\le1/2$ and assume error strictly
below $\varepsilon$ on every such length-$n$ input.  Fix $x,y,z$ with
$t=\ip{x}{y}<\ip{x}{z}=t+d\le\ip{x}{x}$ and
$d\ge2\varepsilon(n-1)/(n-2)$.  Write $L=n-1$, $K=L-1$,
and subtract $t=\ip{x}{y}$ from every position-dependent scalar value.
Here $a$ denotes the self-kernel contribution, which is fixed across all
inputs below, and $g_{x,j}$ includes the affine bias as in
Appendix~\ref{app:gap-domination}.
At source position $j$, put
\[
 b_j=A_j(x,y),\quad c_j=A_j(x,z),\quad
 q_j=g_{x,j}(y)-t,\quad s_j=g_{x,j}(z)-t.
\]
Let $Y$ place $y$ at every source position and $Z$ place $z$ at every source
position.  Let $M_\ell$ place $y$ only at source position $\ell$ and $z$ at
the other source positions.  In the sum over $\ell$, each source position
contributes its $y$ term once and its $z$ term $L-1$ times.  The self term
appears $L$ times on both sides, so the numerators satisfy
\[
 \sum_{\ell=1}^{L}N_{M_\ell}=N_Y+(L-1)N_Z.
\]
Summing the strict upper bounds for the mixed inputs and using the lower
bounds for $Y,Z$ gives
\[
0<(2\varepsilon L-Kd)a
  +2\varepsilon\sum_jb_j
  -K(d-2\varepsilon)\sum_jc_j.
\]
The gap condition makes the self coefficient nonpositive.  Therefore the
effective kernel
\[
 A^{\mathrm{eff}}(x,z)=\sum_jA_j(x,z)
 =\left\langle\phi_Q(x),\sum_j\phi_{K,j}(z)\right\rangle
\]
satisfies Equation~\eqref{eq:app-gap-domination} and has rank at most $r$.
The chain, shell, and rank proofs therefore apply to $A^{\mathrm{eff}}$.
If each source instead has its own factorization
$A_j(x,z)=\ip{\phi_{Q,j}(x)}{\phi_{K,j}(z)}$ of dimension $r_j$,
concatenating these spaces gives rank at most $\sum_jr_j$, which replaces
$r$ in every lower bound.  Placing the query last proves the causal version
over the full visible prefix, with the self token included.  This causal
result applies only to the final query, which can see all source tokens.
\endgroup

\subsection{Distance-sensitive shell counting}
\label{app:shell}

We refine the preceding rank estimate by retaining the distance of each
codeword pair when bounding the Frobenius norm.
Fix an exact length $n\ge3$ and an integer $g\ge2$ with
$\mu=(n-2)(g-1)>1$, and put $\lambda=\mu^{1/g}$.  For the one-head
architecture of Section~\ref{sec:kernel-model}, assume correctness with
error strictly below $1/2$ on every sequence of length exactly $n$.
For $M=2\lfloor m/2\rfloor\ge2g$, define
\[
 V_g(M)=\sum_{a=0}^{g-1}\binom{M/2}{a}^2,
 \qquad c(n,g)=1-\log_2(1+\lambda^{-1}).
\]
We prove the finite bound
\begin{equation}
 r\ge\frac{2^{c(n,g)M}}{\mu^2(M+1)V_g(M)}.
 \label{eq:app-shell-bound}
\end{equation}
\paragraph{Counting distance shells.}
Work on these $M$ coordinates, fixing the remaining coordinate
to zero if $m$ is odd.  Greedy middle-layer packing at directed distance
$g$ gives a code $\cG$ with
\[
 T:=|\cG|\ge
 \frac{\binom{M}{M/2}}{V_g(M)}
 \ge \frac{2^M}{(M+1)V_g(M)}.
\]
We may take $T\ge2$: the complement of the first selected word is at
distance $M/2\ge g$ and survives its deletion step.
Set $B_{S,U}=\alpha(x_S,x_{U^c})$ and normalize each row by $B_{S,S}$.
These diagonal entries are positive by gap domination against any other
codeword.  As in Equation~\eqref{eq:app-chain}, an off-diagonal pair at
directed distance $D=|S\setminus U|$ satisfies
\[
 0\le C_{S,U}<\mu^{-\lfloor D/g\rfloor}
 \le \mu\lambda^{-D},\qquad \lambda=\mu^{1/g}.
\]
The assumptions on $g$ imply the gap condition in
Appendix~\ref{app:gap-domination} for $\varepsilon=1/2$.
At distance $D$ there are at most $\binom{M/2}{D}^2$ codewords in any row.
Including the diagonal using the looser bound $\mu^2>1$ gives
\[
 \begin{aligned}
 \|C\|_F^2
 &\le \mu^2 T\sum_{D=0}^{M/2}\binom{M/2}{D}^2\lambda^{-2D}\\
 &\le \mu^2 T
       \left(\sum_{D=0}^{M/2}\binom{M/2}{D}\lambda^{-D}\right)^2\\
 &=\mu^2 T(1+\lambda^{-1})^M.
 \end{aligned}
\]
The square on the second line contains every term on the first line and
additional nonnegative cross terms.
Since $\operatorname{tr}C=T$ and $\operatorname{rank}C\le r$,
the trace--Frobenius bound gives
\[
 r\ge\frac{T^2}{\|C\|_F^2}
 \ge\frac{2^{[1-\log_2(1+\lambda^{-1})]M}}
          {\mu^2(M+1)V_g(M)}.
\]
For fixed $n,g$, $V_g(M)$ is a polynomial of degree $2g-2$.  The denominator
in the finite bound thus costs only $O(\log M)$ in the exponent.
Taking $n=3,g=5$ proves the asymptotic part of
Theorem~\ref{thm:three-token}.

\paragraph{Choosing the overlap gap.}
For fixed $n$, the exponent increases with $\lambda=[(n-2)(g-1)]^{1/g}$.
We can therefore maximize its natural logarithm
$f_A(g)=\log(A(g-1))/g$, where $A=n-2$ and $g$ is an integer with
$A(g-1)>1$.  For real $g>1$, the derivative has the sign of
$\psi_A(g)=g/(g-1)-\log(A(g-1))$.
Since $\psi_A'(g)=-g/(g-1)^2<0$, $f_A$ increases up to one maximum and
decreases afterwards.  This reduces the integer optimization to the gaps
around that maximum.
For $A=1$, $\psi_1(4)>0$ and $\psi_1(6)<0$, so only $g=4,5,6$ need
comparison.  The inequalities $4^4>3^5$ and $4^6>5^5$ show that $g=5$ is best.
For $A=2$, $\psi_2(3)>0>\psi_2(4)$ places the real maximum between three
and four.  Comparing these two gaps, $4^4>6^3$ shows that three is best.
For $A=3$, $\psi_3(2)>0>\psi_3(3)$ leaves only gaps two and three, and
$6^2>3^3$ again selects three.
For $A\ge4$ and every integer $g\ge3$,
\[
 g-1\le2^{g-2}\le A^{g/2-1},
 \qquad [A(g-1)]^{1/g}\le A^{1/2}.
\]
Thus gap two is optimal for $n\ge6$, tied by gap three at $n=6$.
In particular, writing $c_n=\max_g c(n,g)$,
\[
 c_n=1-\log_2(1+(n-2)^{-1/2})
     =1-\Theta(n^{-1/2})\quad(n\ge6).
\]
This gives the best exponent from the proved shell bound.  The exact minimum
rank remains unknown.  The growing-context result has a separate proof above,
because the asymptotic expression here assumes fixed $n$.

\subsection{Distinct tokens with a unique queried minimizer}
\label{app:unique}

Consider the one-head architecture of Section~\ref{sec:kernel-model}, now
requiring correctness with error $<1/2$ only on inputs of length exactly three
whose tokens are pairwise distinct and whose minimum is unique at the
queried position.  Denominators must still be positive on every such input.
We prove
\[
 r\ge2^{(c_{\mathrm{uniq}}-o(1))m},\qquad
 c_{\mathrm{uniq}}=1-\log_2(1+3^{-1/7})\approx0.11,
\]
and its position-dependent and final-causal versions with the feature
dimensions counted as in Appendix~\ref{app:position-extension}.

\paragraph{Four inputs with distinct tokens.}
Reserve two type coordinates.  For payloads $x,y,z,w\in\bits^{m-2}$, put
\[
 q(x)=(1,1,x),\qquad k_0(w)=(0,0,w),\qquad k_1(w)=(1,0,w).
\]
Let $t=\ip{x}{y}$, $u=\ip{x}{z}=t+d$, where $d\ge3$.
Consider
\[
 \begin{array}{ll}
 Y=(q(x),k_0(y),k_1(y)),&
 Z=(q(x),k_0(z),k_1(z)),\\
 M_0=(q(x),k_0(y),k_1(z)),&
 M_1=(q(x),k_0(z),k_1(y)).
 \end{array}
\]
The source overlaps in these four inputs are respectively
$(t,t+1)$, $(u,u+1)$, $(t,u+1)$, and $(u,t+1)$.
Since $u\le\|x\|_1$, all are below the self-overlap $\|x\|_1+2$.
The gap $d\ge3$ makes each pair's minimum unique, giving targets
$t,u,t,t+1$.  The three different type prefixes make the tokens distinct.
Other positions may have ties.

\paragraph{An effective kernel preference.}
Write $q=q(x)$, $a=\alpha(q,q)$,
$b_i=\alpha(q,k_i(y))$, and $c_i=\alpha(q,k_i(z))$.
After subtracting $t$ from the scalar values, the numerators satisfy
$N_{M_0}+N_{M_1}=N_Y+N_Z$ exactly.  Strict error below $1/2$ and positive
denominators imply
\[
 0<\tfrac12D_{M_0}+\tfrac32D_{M_1}
       +\tfrac12D_Y-(d-\tfrac12)D_Z
   =(3-d)a+b_0+2b_1-(d-2)c_0-(d-1)c_1.
\]
The self coefficient is nonpositive.  Dropping it gives
$b_0+2b_1>(d-2)c_0+(d-1)c_1$.
Because $d\ge3$ implies $d-2\ge(d-1)/2$ and $c_0\ge0$, the right side
is at least $(d-1)(c_0+2c_1)/2$.  Thus the weighted effective kernel satisfies
\[
 \bar\alpha(x,y)>\frac{d-1}{2}\bar\alpha(x,z),
 \qquad
 \bar\alpha(x,w)=\alpha(q(x),k_0(w))+2\alpha(q(x),k_1(w)).
\]
The effective kernel is nonnegative and has a factorization through the
original $r$-dimensional query space, using the key feature
$\phi_K(k_0(w))+2\phi_K(k_1(w))$.  The model may use the type coordinates
arbitrarily.

\paragraph{Applying the shell bound.}
Take $M=2\lfloor(m-2)/2\rfloor$ active coordinates and gap seven.
Every overlap gap $d\ge7$ gives a factor $(d-1)/2\ge3$ for
$\bar\alpha$.  Each comparison uses the four distinct-token inputs above,
so correctness on the restricted input family suffices for every link.
The chain and shell calculation applied to $\bar\alpha$ therefore gives
\[
 r\ge
 \frac{2^{[1-\log_2(1+3^{-1/7})]M}}
      {9(M+1)V_7(M)}
 \qquad(M\ge14).
\]
Since $m-M\le3$, this proves the stated exponent.

For position-dependent maps, put type-zero and type-one keys at their two
fixed source positions in all four inputs.  The same numerator
cancellation gives the weighted effective kernel
$A_1(q(x),k_0(w))+2A_2(q(x),k_1(w))$.  Its rank is at most $r$ for a
shared query space, or at most the sum of independent position dimensions
after concatenation.  Put the query last to obtain the causal version.

\section{Multiple heads and polynomial decoding: complete proofs}
\label{app:multihead}

We first derive the uniform domination step and use it to prove
Theorem~\ref{thm:multihead}.  We then extend the argument to positional maps
and prove the one-head hierarchy in Theorem~\ref{thm:polynomial}.
All bounds in this section apply to one attention layer, nonnegative kernel
entries, positive denominators, arbitrary finite-dimensional tokenwise values,
and no additional cross-token channel.  Head values and output coefficients
can be signed.  The first three subsections use the affine readout in
Equation~\eqref{eq:multihead-model}.  The final subsection uses one
head with polynomial decoding of fixed degree.
Arithmetic is exact throughout.  Complements below are taken within the
active set of $M=2\lfloor m/2\rfloor$ coordinates, with any remaining
coordinate fixed to zero.

\subsection{Finite differences force a uniformly dominating head}
\label{app:uniform-head}

Assume the position-independent model in Equation~\eqref{eq:multihead-model}
has error strictly below $1/2$ on every input of exact length $H+2$.
For $H\ge1$, put $L=H+1$ and define the integer
\begin{equation}
 g_H=\left\lceil
 \frac{1+(2^L-1)(1+2L)^H}{2}
 \right\rceil.
 \label{eq:head-gap}
\end{equation}
These constants are not optimized: $g_1=8$, $g_2=172$, and $g_3=5468$.
Section~\ref{sec:rank} gives a sharper bound when $H=1$.
We will show that if $\ip{x}{y}+g_H\le\ip{x}{x}$, some head
$h=h(x,y)$ satisfies $\alpha_h(x,y)>2\alpha_h(x,z)$ for every Boolean
key $z$ with $\ip{x}{z}\ge\ip{x}{y}+g_H$.  This is the domination step
needed for Theorem~\ref{thm:multihead}.  The chosen head depends on $x,y$
and must work for all such $z$ simultaneously.

Fix $x,y$, set $t=\ip{x}{y}$, and suppose $t+g_H\le\ip{x}{x}$.
If the conclusion fails, for each head $h$ choose a Boolean token $z_h$
with $\ip{x}{z_h}\ge t+g_H$ and
$\alpha_h(x,y)\le2\alpha_h(x,z_h)$.
Use $L$ source tokens $(z_1,\ldots,z_H,z_1)$ in addition to the fixed
query token $x$.  For each $A\subseteq[L]$, replace the source positions
in $A$ by $y$, giving an input $X_A$ of the same exact length $H+2$.
The base target $t_x(X_\varnothing)$ is at least $t+g_H$, including
the self token.  Every nonempty replacement has target exactly $t$.

For fixed $x$, combine the value coordinates of head $h$ using its readout weights:
$u_h(z)=w_h(x)^\top v_h(z)$.  Let $N_h(A)$ be its weighted scalar
numerator and $D_h(A)>0$ its denominator, including the self contribution.
If $s_j$ is the token in base source slot $j$, then
\[
 D_h(A)=D_h(\varnothing)+\sum_{j\in A}
       \bigl(\alpha_h(x,y)-\alpha_h(x,s_j)\bigr).
\]
The numerator has the same form, with each kernel contribution multiplied
by the corresponding scalar value $u_h$.  Thus both are affine in the
$L$ indicators of membership in $A$.  With
\[
 \begin{aligned}
 Q(A)&=\prod_{h=1}^H D_h(A),\\
 P(A)&=(F_x(X_A)-t)Q(A)\\
     &=(\beta(x)-t)\prod_hD_h(A)
       +\sum_h N_h(A)\prod_{k\ne h}D_k(A)
 \end{aligned}
\]
each term of $P(A)$ contains at most $H$ affine factors, so its degree is
at most $H$.  On Boolean indicators, every positive power of an indicator
equals that indicator.  Reducing these powers gives a multilinear
polynomial of the same or lower degree.  Because $H<L$, every monomial
omits an indicator.  The terms with that indicator zero and one then have
opposite signs in the alternating sum and cancel.  Thus
\[
 P(\varnothing)=-\sum_{\varnothing\ne A\subseteq[L]}(-1)^{|A|}P(A).
\]
Error strictly below $1/2$ and positive $Q(A)$ give
$P(\varnothing)>(g_H-\tfrac12)Q(\varnothing)$ and
$|P(A)|<Q(A)/2$ for nonempty $A$.  Applying the triangle inequality to
the alternating sum gives
\begin{equation}
 (g_H-\tfrac12)Q(\varnothing)
 <\tfrac12\sum_{\varnothing\ne A\subseteq[L]}Q(A).
 \label{eq:alternating-gap}
\end{equation}

The base denominator of head $h$ contains $\alpha_h(x,z_h)$.
Replacing positions removes only nonnegative contributions and adds at
most $L$ copies of $\alpha_h(x,y)$.  Thus
\[
 D_h(A)\le D_h(\varnothing)+L\alpha_h(x,y)
          \le(1+2L)D_h(\varnothing).
\]
It follows that $Q(A)\le(1+2L)^H Q(\varnothing)$.  There are $2^L-1$
nonempty subsets.  Substituting this bound into
Equation~\eqref{eq:alternating-gap} and dividing by the positive
$Q(\varnothing)$ gives
\[
 2g_H-1<(2^L-1)(1+2L)^H\le2g_H-1,
\]
where the last inequality is the definition of $g_H$.
This contradiction proves uniform domination.

\subsection{Amplifying along head labels}
\label{app:head-chain}

\begin{proof}[Proof of Theorem~\ref{thm:multihead}]
We amplify uniform head domination along chains of increasing overlap.
For an explicit finite bound, suppose $M\ge168$ and
$\Delta=\lfloor M/8\rfloor\ge g_H$, and set
\[
 K=\lfloor\Delta/g_H\rfloor,\qquad q=\lceil K/H\rceil.
\]
The middle-layer packing in Appendix~\ref{app:packing} gives a code
$\cG\subseteq\binom{[M]}{M/2}$ of directed distance at least $\Delta$
and size $T\ge2^{M/10}$.
For each $S\in\cG$ choose nested sets $P_{S,k}\subseteq S$ of size
$kg_H$, for $k=0,\ldots,K-1$, and define keys
\[
 y_{S,k}=x_{S^c\cup P_{S,k}}.
\]
At query $x_S$, their overlaps are $kg_H$.  Because
$Kg_H\le\Delta\le M/2$, the domination result in Appendix~\ref{app:uniform-head}
assigns each level a head that dominates all keys at least one gap above it.

Some head $h_S$ occurs at least $q$ times among the $K$ labels.
Choose its first $q$ occurrences, at levels $k_1<\cdots<k_q$, and put
$k_S=k_1$ and $\kappa_S=y_{S,k_S}$.
Between successive chosen occurrences, overlaps increase by at least $g_H$.
Uniform domination therefore gives $q-1$ comparisons with a factor greater
than two, all using head $h_S$.  The levels between these occurrences can
have different head labels.
For another row $U\ne S$,
\[
 \ip{x_S}{\kappa_U}
 =|S\setminus U|+|S\cap P_{U,k_U}|
 \ge\Delta\ge Kg_H.
\]
The last chosen level has overlap at most $(K-1)g_H$, so its comparison
with $\kappa_U$ gives one more factor greater than two.  The $q-1$
internal comparisons and this final comparison give
\[
 \alpha_{h_S}(x_S,\kappa_S)
 >2^q\alpha_{h_S}(x_S,\kappa_U).
\]
Since $T>1$ and all kernel entries are nonnegative, this strict inequality
also proves that each diagonal entry is positive.

For each head let $A_h=(\alpha_h(x_S,\kappa_U))_{S,U\in\cG}$, and let
$R_h$ be the diagonal zero--one matrix selecting rows assigned to $h$.
The feature factorization gives $\operatorname{rank}A_h\le r_h$, and
selecting rows cannot increase rank.
The matrix that uses the chosen head in each row is
\[
 B=\sum_hR_hA_h,\qquad
 B_{S,U}=\alpha_{h_S}(x_S,\kappa_U),\qquad
 \operatorname{rank}B\le\sum_hr_h.
\]
Each column uses the same key $\kappa_U$ for every row, even though the
chosen head can vary across rows.
Divide each row by its positive diagonal entry and apply
the bound in Appendix~\ref{app:approx-identity} to obtain
\begin{equation}
 \sum_hr_h\ge\frac{T}{1+(T-1)4^{-q}}
 \ge2^{\min\{M/10,\,2\lceil\lfloor\Delta/g_H\rfloor/H\rceil\}-1}.
 \label{eq:finite-multihead}
\end{equation}
For fixed $H$, the exponent is a positive constant times $m$, up to
an additive constant.  All constructed inputs have length exactly $H+2$.
\end{proof}

\subsection{Position-dependent maps and a final causal query}
\label{app:positions}

We prove the total-rank conclusion of Theorem~\ref{thm:multihead} for
arbitrary position-dependent tokenwise maps and a fixed queried position,
with the same nonnegative kernels, positive denominators, affine mixing,
and error strictly below $1/2$ on every input of exact length $H+2$.

\paragraph{Products of head kernels.}
Fix the queried position, let $L=H+1$, and write its source-position
kernels as
\[
 A_{h,j}(x,z)=
 \ip{\phi_{Q,h}(x)}{\phi_{K,h,j}(z)}\ge0
 \qquad(j\in[L]).
\]
The self contribution $a_h(x)\ge0$ stays fixed in all replacements.
Values and their scalar numerators can depend arbitrarily on the fixed
source positions, so they remain affine in replacement indicators.
For every nonempty $I\subseteq[H]$ define
\[
 K_{I,j}(x,z)=\prod_{h\in I}A_{h,j}(x,z).
\]
Put
\[
 J=L(2^H-1),\qquad C=2(2^H-1),\qquad
 \widetilde g_H=
 \left\lceil\frac{1+(2^L-1)C^H}{2}\right\rceil.
\]
Fix $x,y$ with $\ip{x}{y}+\widetilde g_H\le\ip{x}{x}$, and put
$t=\ip{x}{y}$.  We claim that some label $(I,j)$ satisfies
$K_{I,j}(x,y)>2K_{I,j}(x,z)$ for every $z$ with overlap at least
$t+\widetilde g_H$.

Otherwise, choose a violating key $z_{I,j}$ for every label.
At each position $j$, independently draw $Z_j$ by choosing one of its
$2^H-1$ indexed witnesses with equal probability.  Repeated witness tokens
keep their multiplicity.  The expectation therefore contains each
$K_{I,j}(x,z_{I,j})$ with weight at least $1/(2^H-1)$.
Nonnegativity gives, for every $I,j$,
\[
 K_{I,j}(x,y)\le2K_{I,j}(x,z_{I,j})
       \le C\,\mathbb E K_{I,j}(x,Z_j).
\]
For $A\subseteq[L]$, write $Z_j^A=y$ if $j\in A$ and $Z_j^A=Z_j$ otherwise.
Let $X_A$ keep the query at its fixed position and place $Z_j^A$ at source
position $j$.  For each draw, the base input has target at least
$t+\widetilde g_H$ and every nonempty replacement has target $t$.
The same degree-$H$ cancellation therefore proves
Equation~\eqref{eq:alternating-gap} for that input with gap
$\widetilde g_H$.

Expand the nonnegative product
$Q(A)=\prod_h(a_h+\sum_j A_{h,j}(x,Z_j^A))$.
Each term assigns every head to either the self contribution or a source
position.  Write $I_0$ for the heads assigned to the self contribution and
$I_j$ for those assigned to source position $j$.  In the base input, the
expected value of this term is
\[
 \left(\prod_{h\in I_0}a_h\right)
 \prod_{j:I_j\ne\varnothing}\mathbb E K_{I_j,j}(x,Z_j).
\]
Heads assigned to the same position share its sampled token and are
grouped inside $K_{I_j,j}$.  Independence across positions gives the
product of expectations.  Replacing a used position $j$ by $y$ changes
its factor to $K_{I_j,j}(x,y)$, which is at most $C$ times the base
expectation.  The other factors are unchanged.  Each term uses at most
$H$ source positions, so its expectation is at most $C^H$ times its base
expectation.
Summing these nonnegative terms gives
\[
 \mathbb E Q(A)\le C^H\mathbb E Q(\varnothing).
\]
Every base denominator product is positive, and there are finitely many
sampled inputs.  Averaging Equation~\eqref{eq:alternating-gap} therefore
preserves strictness.  Substituting the expectation bound and canceling
$\mathbb E Q(\varnothing)>0$ gives
$2\widetilde g_H-1<(2^L-1)C^H$, contradicting the definition of
$\widetilde g_H$.  This proves the claimed bound for one label and all
keys above the gap.

\paragraph{Amplification and total rank.}
Use the code and chain construction from Appendix~\ref{app:head-chain},
with gap $\widetilde g_H$ and $J$ labels in place of heads.  If $M\ge168$ and
$\Delta=\lfloor M/8\rfloor\ge\widetilde g_H$, put
\[
 \widetilde q=
 \left\lceil\frac{\lfloor\Delta/\widetilde g_H\rfloor}{J}\right\rceil,
 \qquad
 L_B=2^{\min\{M/10,2\widetilde q\}-1}.
\]
Each row $S$ chooses a label $(I_S,j_S)$ occurring at least $\widetilde q$
times, and $\kappa_S$ is the key at its first chosen level.  The matrix
\[
 B_{S,U}=K_{I_S,j_S}(x_S,\kappa_U)
\]
satisfies $B_{S,S}>2^{\widetilde q}B_{S,U}$ for $S\ne U$ by the same
comparisons.  Row normalization and the trace--Frobenius estimate give
$\operatorname{rank}B\ge L_B$.
For each label, the product kernel factors as
\[
 K_{I,j}(x,z)=
 \left\langle\bigotimes_{h\in I}\phi_{Q,h}(x),
                 \bigotimes_{h\in I}\phi_{K,h,j}(z)\right\rangle.
\]
These tensor vectors have $\prod_{h\in I}r_h$ coordinates.  As in the
head-label argument, selecting rows from each kernel matrix and summing
over all labels gives
\[
 L_B\le\operatorname{rank}B
 \le\sum_{j=1}^L\sum_{\varnothing\ne I\subseteq[H]}
                         \prod_{h\in I}r_h
 =L\left[\prod_h(1+r_h)-1\right].
\]
With $R=\sum_h r_h$, the arithmetic--geometric mean inequality gives
$\prod_h(1+r_h)\le(1+R/H)^H$.  Substituting into the preceding display
and solving for $R$ yields
\begin{equation}
 R\ge H\left[\left(1+\frac{L_B}{L}\right)^{1/H}-1\right]
     =2^{\Omega_H(m)}.
 \label{eq:positional-multihead}
\end{equation}

\paragraph{Independent position spaces and a final causal query.}
If a head has independent query and key spaces at different source positions,
concatenate their query features and place each position's key features in
its own block.  Counting the resulting total dimension in $r_h$ gives the
same bound and includes every position's feature space.
Finally, place the fixed query after all $L$ source positions for a causal
model.  Its entire prefix and self token are visible, and the proof is
unchanged.  This argument only covers queries that can see all sources.

\subsection{The one-head degree/context hierarchy}
\label{app:polynomial}

\begin{proof}[Proof of Theorem~\ref{thm:polynomial}]
This proof uses one position-independent normalized nonnegative head, with
positive denominators, arbitrary finite-dimensional tokenwise values, and
a query-dependent polynomial decoder of total degree at most $D\ge1$.
There is one attention layer, no other cross-token channel, and exact real
arithmetic.  The degree $D$ is fixed independently of $m$.

\paragraph{Rank one at exact length $D+1$.}
Put $n=D+1$, take the constant kernel $\alpha(x,z)=1$, and use the standard
basis values $v(z)=e_z\in\mathbb R^{2^m}$, with coordinates indexed by
Boolean tokens.  The normalized
output $p$ is the empirical token distribution, and its denominator is
$n>0$.  For a fixed query $x$ and $1\le k\le\|x\|_1$, define the
linear function
\[
 A_k(p)=\sum_{z:\,\ip{x}{z}\ge k}p_z.
\]
The count $s_k=nA_k(p)$ is an integer in $\{1,\ldots,n\}$: the query
token itself meets every threshold in this range.  Therefore
\[
 J_n(s)=\frac{1}{(n-1)!}\prod_{j=1}^{n-1}(s-j)
\]
equals zero at $s=1,\ldots,n-1$, because one factor vanishes.  At $s=n$,
the product is $(n-1)!$, so the value is one.  The decoder
\[
 F_x(p)=\sum_{k=1}^{\|x\|_1}J_n\bigl(nA_k(p)\bigr)
\]
has total degree at most $n-1=D$, since each $A_k$ is linear.
If the minimum overlap is $t$, the thresholds met by every token are
exactly $k=1,\ldots,t$.  The sum therefore has exactly $t$ terms equal
to one and all remaining terms equal to zero, giving the minimum overlap.
For $x=0$ the sum is empty and correctly returns zero.  Repetitions and
ties do not affect the argument.  The kernel has feature rank one, but
the chosen value dimension is $2^m$.  This construction uses its fixed
exact length $n$.  It does not claim that the same polynomial is exact
at every shorter length.

\paragraph{Exponential rank at exact length $D+2$.}
Now assume error strictly below $1/2$ on every input of exact length
$D+2$.  Set $L=D+1$ and define the integer gap
\[
 g_D=\left\lceil
 \frac{1+(2^L-1)(1+2L)^D}{2}
 \right\rceil.
\]
We first prove that whenever $t=\ip{x}{y}$ and
$t+g_D\le\ip{x}{x}$,
\begin{equation}
 \alpha(x,y)>2\alpha(x,z)
 \quad\text{for every }z\text{ with }\ip{x}{z}\ge t+g_D.
 \label{eq:polynomial-domination}
\end{equation}
Suppose instead that such a $z$ has
$b=\alpha(x,y)\le2c=2\alpha(x,z)$, and write
$a=\alpha(x,x)$.  Start from $(x,z^L)$ and, for every
$A\subseteq[L]$, replace the source positions in $A$ by $y$.
All inputs $X_A$ have the same exact length $D+2$.
Their targets are at least $t+g_D$ when $A=\varnothing$ and
exactly $t$ otherwise, including the self token in both cases.

Write $B(A)>0$ for the normalization denominator and $N_j(A)$ for
each value-coordinate numerator.  Both are affine in the $L$ Boolean
replacement indicators.  A decoder monomial of total degree $e\le D$
becomes a product of $e$ numerator factors and $D-e$ denominator
factors after multiplication by $B(A)^D$.  There are $D$ affine factors
in total, so the resulting degree is at most $D$.  This includes constant
decoder terms, and the subtracted term $tB(A)^D$ has the same degree bound.
Thus
\[
 P(A)=\bigl(F_x(X_A)-t\bigr)B(A)^D
\]
has Boolean multilinear degree at most $D<L$.  Its order-$L$
alternating difference is zero by the cancellation argument in
Appendix~\ref{app:uniform-head}.  Error strictly below $1/2$ and the
triangle inequality give
\[
 (g_D-\tfrac12)B(\varnothing)^D
 <\tfrac12\sum_{A\ne\varnothing}B(A)^D.
\]
On the other hand, nonnegativity and $b\le2c$ imply
\[
 B(A)\le B(\varnothing)+Lb
 \le(1+2L)B(\varnothing),
\]
because $B(\varnothing)=a+Lc\ge c$.  Substituting into the previous
inequality and counting the $2^L-1$ nonempty subsets gives
$2g_D-1<(2^L-1)(1+2L)^D$, contradicting the definition of $g_D$.
The proof allows values of any finite
width, with either sign for values and polynomial coefficients.

For completeness, amplify Equation~\eqref{eq:polynomial-domination}
using the same constant-weight packing as the affine theorem.  Let
$M=2\lfloor m/2\rfloor\ge168$,
$\Delta=\lfloor M/8\rfloor\ge g_D$, and
$K=\lfloor\Delta/g_D\rfloor$.  Appendix~\ref{app:packing} gives
$\cG\subseteq\binom{[M]}{M/2}$ of directed distance at least
$\Delta$ and size $T\ge2^{M/10}$.  For each $S\in\cG$, take
nested $P_{S,k}\subseteq S$ of size $kg_D$, for
$k=0,\ldots,K-1$, with $P_{S,0}=\varnothing$.
The keys $x_{S^c\cup P_{S,k}}$ have overlaps $kg_D$ with query
$x_S$.  Each next level is at least one gap higher.  Every other
column key $x_{U^c}$ has overlap
$|S\setminus U|\ge\Delta\ge Kg_D$.
Every comparison satisfies the required gap below the self-overlap because
$Kg_D\le\Delta\le M/2$.  There are $K-1$ steps between the nested keys
and one final comparison to $x_{U^c}$.  Chaining these $K$ strict
preferences gives
\[
 \alpha(x_S,x_{S^c})>2^K\alpha(x_S,x_{U^c})
 \qquad(S\ne U).
\]
In particular the diagonal of
$B_{S,U}=\alpha(x_S,x_{U^c})$ is positive.  This matrix has rank at
most $r$.  Dividing each row by its diagonal gives a matrix
$C$ with diagonal entries one and off-diagonal entries below $2^{-K}$.
The nonsymmetric trace--Frobenius inequality from
Appendix~\ref{app:approx-identity} gives
\[
 r\ge\frac{T}{1+(T-1)4^{-K}}
 \ge2^{\min\{M/10,\,2\lfloor\Delta/g_D\rfloor\}-1}
 =2^{\Omega_D(m)}.
\]
For odd $m$, fix the final coordinate to zero.  The lower bound is at one
exact length and allows model parameters to depend on that length.
It does not assert a threshold for a single model required to work at all
shorter lengths.
\end{proof}

\section{Exact-real escapes and finite-alphabet information}
\label{app:information}

We verify the two constructions from Section~\ref{sec:multilevel}, then
prove the transcript bound and the size estimate for its hard family.

\subsection{A scalar construction uniform over all finite lengths}
\label{app:uniform-real}

For fixed $m$, we construct a position-independent positive rank-one head
with scalar values and a query-dependent exact-real decoder that computes
Min-IP at every finite length.  We encode the empirical token distribution
using maps independent of length.  Enumerate the Boolean tokens as
$z_0,\ldots,z_{2^m-1}$.
Choose a transcendental number
$\theta\in(0,1)$, and use constant features $\phi_Q=\phi_K=1$ and values
$v(z_k)=\theta^k$.  On every nonempty length-$N$ input the denominator
is $N>0$, and all queries receive the scalar
\[
 a(X)=\frac1N\sum_{j=1}^N v(x_j)
     =\sum_{k=0}^{2^m-1}p_k(X)\theta^k,
 \qquad
 p_k(X)=\frac{|\{j:x_j=z_k\}|}{N}.
\]
Suppose two inputs $X,X'$, possibly of different lengths, give the same scalar.
Subtracting the two expressions gives a polynomial with rational
coefficients $p_k(X)-p_k(X')$ that is zero at $\theta$.  A transcendental
number cannot be a root of a nonzero polynomial with rational coefficients.
Thus every coefficient is zero, so the empirical distributions are identical.
In particular, the same tokens appear in both inputs.

For scalar values produced by inputs, define
\[
 D_x(a(X))=\min_{k:p_k(X)>0}\ip{x}{z_k},
\]
and assign arbitrary values to $D_x$ elsewhere.  The uniqueness just proved
makes this well defined.  For every query token it returns the
exact minimum, including self tokens, repetitions, and ties.
The kernel, scalar value map, and decoder depend on $m$ but not on $N$.
The decoder need not recover the token counts or know the length.

For a nonzero query $x$, an all-$x$ input has scalar $v(x)$ and target
$\|x\|_1>0$.  An input with $N$ copies of $x$ and one zero token has scalar
$(Nv(x)+v(0))/(N+1)\to v(x)$ but target zero.
Thus this decoder is discontinuous even on scalar values produced by inputs.
The construction does not provide a bounded-size neural decoder or a stable
finite-precision procedure for recovering the distribution.

\subsection{A small nonlinear decoder at length three}
\label{app:relu-escape}

At exact length three, a positive rank-one head with $2m$ value coordinates
and a query-dependent ReLU readout can compute Min-IP at the final query.
For source tokens $y,z$ and final query $x$, use the uniform kernel and
position-dependent values
\[
 v_1(y)=(y,0),\qquad v_2(z)=(0,z),\qquad v_3(x)=0.
\]
The normalized head output is $(y,z)/3$, since the self token is included
with unit weight.  Taking inner products of this output with $(3x,0)$ and
$(0,3x)$ recovers $u=\ip{x}{y}$ and $v=\ip{x}{z}$.  Then
\[
 \min(u,v)=u-\operatorname{ReLU}(u-v).
\]
Both candidate overlaps are at most $\ip{x}{x}$ for Boolean tokens,
so this is also the self-attention minimum.  The denominator is exactly
three.  This construction uses a continuous nonlinear readout and different
value maps at the two source positions.  Together with the scalar
construction, it shows why decoder assumptions matter in a feature-rank
lower bound.

\subsection{Finite-precision transcript information}

\begin{proof}[Proof of Theorem~\ref{thm:transcript}]
We first construct the hard family with fixed query tokens and positions.

\paragraph{The hard family.}
Let $m=2+M+s$ and let $\cC\subseteq\binom{[M]}h$ have directed distance
at least $s+1$, with $M,s\ge1$ and $1\le h\le M$.
Complements in this construction are taken within $[M]$.
For $S,U\in\cC$ and $k\in\{0,\ldots,s\}$, set
\[
 q_S=(0,1,x_{S^c},\mathbf1_s),
 \qquad d_{U,k}=(1,0,x_U,u_k),
\]
where $u_k\in\bits^s$ has Hamming weight $k$.
Choose an integer $1\le q\le|\cC|$ with $2q\le n$, and fix an ordering of
a $q$-word subcode $\cQ\subseteq\cC$.
For each assignment $\kappa\in\{0,\ldots,s\}^{\cQ}$, the input $X_\kappa$
consists of the fixed query block $(q_S)_{S\in\cQ}$ followed by the
database block $(d_{S,\kappa_S})_{S\in\cQ}$.  Its length is $2q\le n$.
The query--database overlap is
\[
 \ip{q_S}{d_{U,k}}=|U\setminus S|+k.
\]
The matching token has overlap $k$, every nonmatching database token has
overlap at least $s+1$, and
$\ip{q_S}{q_T}=1+|S^c\cap T^c|+s\ge s+1$.  Hence the query target is exactly
$\kappa_S$, including the query's self token.  No padding is needed, and
the minimum is attained at the unique matching database token.

\paragraph{Transcript injectivity.}
For assignment $\kappa$, let
$\tau(\kappa)=(\sigma_1(X_\kappa),\ldots,\sigma_L(X_\kappa))$.  If two
assignments $\kappa,\kappa'$ have equal transcripts, consider a fixed query
$q_S$ at its fixed position $i_S$.  Its initial state $E(q_S,i_S)$ is the
same in both inputs.  Write its deterministic local update at layer $\ell$ as
\[
 h_{i_S}^{(\ell)}=
 F_\ell\bigl(h_{i_S}^{(\ell-1)},i_S,\sigma_\ell\bigr).
\]
Equal states at layer $\ell-1$ and equal sketches at layer $\ell$ give equal
next states.  Induction gives equal final states and tokenwise outputs at
every $q_S$.  With error strictly below $1/2$, rounding this shared output
gives both $\kappa_S$ and $\kappa'_S$, so these integers are equal.
Thus the transcript map is injective on
$(s+1)^q$ assignments, and
\[
 \prod_{\ell=1}^L|\Sigma_\ell|\ge(s+1)^q.
\]
Taking base-two logarithms proves the theorem.

\paragraph{Counting head sketches.}
For the numerator and normalizer sketches considered in
Equation~\eqref{eq:finite-precision}, a head communicates $rd_v$ numerator
entries and $r$ normalizer entries,
for a total of $r(d_v+1)$ coordinates.  At most $2^p$ values per coordinate
give at most $2^{r(d_v+1)p}$ head sketches, and hence at most
$2^{LHr(d_v+1)p}$ transcripts across $H$ heads and $L$ layers.
Counting these combinations requires no independence.  Comparing with
the injectivity bound gives Equation~\eqref{eq:finite-precision}.

Local state dimension and deterministic local computation are unrestricted.
Different channels or layers may have different finite alphabets.
Equation~\eqref{eq:finite-precision} uses equal coordinate budgets.
Every additional cross-token message must be included, with its own
alphabet, in the transcript count.
\end{proof}

\subsection{Size of the multilevel family}
\label{app:multilevel-code}

For $m\ge32$, put $s=\lfloor m/\log_2m\rfloor$, $M=m-s-2$, and
$h=\lfloor M/2\rfloor$.  We will construct a constant-weight code satisfying
$|U\setminus S|\ge s+1$ for distinct words and
\[
 |\cC|\ge
 \frac{2^M}{(M+1)(s+1)(eM/s)^{2s}}
 =2^{m-O(m\log\log m/\log m)}.
\]
For $m\ge32$, $s\le m/5$ and $M\ge4m/5-2$, so
$h\ge M/2-1\ge2m/5-2>s+1$.  In particular, the deletion radius $s$ is
smaller than both $h$ and $M-h$.

Greedily pack $\binom{[M]}h$, deleting after each selection all words at
directed distance at most $s$.  Every later selected word then has distance
at least $s+1$ from all earlier selections.  Equal weights make the two
directed distances equal.  To obtain a word at distance $a$ from a fixed
word, remove $a$ of its $h$ coordinates and add $a$ of the other $M-h$.
There are therefore $\binom ha\binom{M-h}{a}$ such words.
For $1\le a\le s$, the binomial bound $\binom ba\le(eb/a)^a$ gives
\[
 \binom ha\binom{M-h}{a}
 \le(eM/a)^{2a}\le(eM/s)^{2s}.
\]
The last inequality follows because the function $a\log(eM/a)$ has
derivative $\log(M/a)\ge0$ for $0<a\le M$.  The distance-zero term is one
and satisfies the same bound.  Thus one selection deletes at most
\[
 \sum_{a=0}^s\binom ha\binom{M-h}{a}
 \le(s+1)(eM/s)^{2s}.
\]
When the process ends, all words have been deleted, so the number selected
is at least the layer size divided by this deletion bound.
The coefficient $\binom Mh$ is a largest binomial coefficient, and the
$M+1$ coefficients sum to $2^M$.  Hence the layer size is at least
$2^M/(M+1)$, proving the displayed finite bound.

To obtain its asymptotic form, note that $s=\Theta(m/\log m)$ and
$M/s=\Theta(\log m)$.  Taking base-two logarithms of the finite bound gives
\[
 \log_2|\cC|\ge
 m-(s+2)-\log_2(M+1)-\log_2(s+1)-2s\log_2(eM/s).
\]
The last term is $O(m\log\log m/\log m)$, and the other losses are
$O(m/\log m+\log m)$.  This gives
$|\cC|=2^{m-O(m\log\log m/\log m)}$.
For $n\ge2$, choose
$q=\min\{\lfloor n/2\rfloor,|\cC|\}$.  Since
$\log_2(s+1)=\Theta(\log m)$, Theorem~\ref{thm:transcript} gives
$\Omega(q\log m)$ transcript bits.  When
$\lfloor n/2\rfloor\le|\cC|$, we have $q=\lfloor n/2\rfloor\ge n/3$,
which gives $\Omega(n\log m)$.

\end{document}